%% file: root.tex
\documentclass[letterpaper, 10 pt, conference]{ieeeconf}  

\IEEEoverridecommandlockouts                              

\title{\LARGE \bf
Zero-Shot Generalization from Motion Demonstrations to New Tasks
}
\author{Kilian Freitag$^{1*}$, Alvin Combrink$^{1*}$, and Nadia Figueroa$^{2}$
    \thanks{$^{*}$Equal contribution. $^{1}$Department of Electrical Engineering, Chalmers University of Technology, Gothenburg, Sweden {\tt\small \{combrink,kilian.freitag\}@chalmers.se}. $^{2}$GRASP Lab, University of Pennsylvania, PA 19104, USA. This work is supported by the ITEA project ArtWork.}
}

\usepackage{amsmath}
\usepackage{todonotes}
\usepackage{amssymb}
\usepackage{hyperref}
\usepackage{graphicx}
\usepackage{subcaption}
\usepackage{booktabs}
\usepackage{balance}
\usepackage{xspace}
\newtheorem{theorem}{Theorem}[section]
\newtheorem{definition}{Definition}[section]

\begin{document}
\maketitle
\thispagestyle{empty}
\pagestyle{empty}

\input{sections/NewCommands}

\begin{abstract}
    \input{sections/Abstract}
\end{abstract}

\section{Introduction}
\input{sections/Introduction}

\section{Problem Description and Preliminaries}\label{sec:preliminaries}

\input{sections/Problem_Description}

\section{The Gaussian Graph}

\input{sections/GaussianGraph}

\section{Demonstration Stitching}
\input{sections/TimeInvariantStitching}

\section{Demonstration Chaining}\label{sec:chaining}

\input{sections/Chaining}

\section{Results}

\input{sections/Experiments}

\section{Discussions}

\input{sections/Discussions}

\section{Conclusions}

\input{sections/Conclusion}

\bibliographystyle{ieeetr}
\bibliography{Bib}

\end{document}

%% file: sections/NewCommands.tex
\newcommand{\Real}{\mathbb{R}}

\newcommand{\pos}{\xi}
\newcommand{\vel}{\dot{\xi}}

\newcommand{\DemoSet}{\mathbb{D}}
\newcommand{\demo}{\mathcal{D}}

\newcommand{\Params}{\Theta}
\newcommand{\param}{\theta}
\newcommand{\gmean}{\mu}
\newcommand{\gcov}{\mathbf{\Sigma}}
\newcommand{\gprior}{\pi}
\newcommand{\gpost}{\gamma}

\newcommand{\dsA}{\mathbf{A}}
\newcommand{\dsP}{\mathbf{P}}
\newcommand{\dsQ}{\mathbf{Q}}

\newcommand{\Graph}{\mathcal{G}}
\newcommand{\Nodes}{\mathcal{V}}
\newcommand{\Edges}{\mathcal{E}}
\newcommand{\Weights}{\mathit{W}}
\newcommand{\direction}{\psi}
\newcommand{\paramdist}{{\eta_\mathit{dist}}}
\newcommand{\paramdir}{{\eta_\mathit{dir}}}
\newcommand{\parambhat}{{\eta_\mathit{BC}}}
\newcommand{\ggSolution}{\sigma}

\newcommand{\Gaussians}{\mathbf{G}}
\newcommand{\gaussian}{g}

\newcommand{\DSC}{\mathcal{C}}
\newcommand{\DSseq}{\mathcal{F}}
\newcommand{\TriggerSeq}{\Gamma}
\newcommand{\trigger}{\gamma}
\newcommand{\TransSeq}{\mathcal{T}}
\newcommand{\trans}{\tau}
\newcommand{\InterDSseq}{\mathcal{U}}
\newcommand{\inter}{u}

\newcommand{\datasetSmall}{\texttt{2D\,Small}\xspace}
\newcommand{\datasetLarge}{\texttt{2D\,Large}\xspace}
\newcommand{\datasetPcgmm}{\texttt{3D\,PC-GMM}\xspace}

%% file: sections/Abstract.tex
Learning motion policies from expert demonstrations is an essential paradigm in modern robotics.
While end-to-end models aim for broad generalization, they require large datasets and computationally heavy inference. 
Conversely, learning dynamical systems (DS) provides fast, reactive, and provably stable control from very few demonstrations. 
However, existing DS learning methods typically model isolated tasks and struggle to reuse demonstrations for novel behaviors. 
In this work, we formalize the problem of combining isolated demonstrations within a shared workspace to enable generalization to unseen tasks. 
The \emph{Gaussian Graph} is introduced, which reinterprets spatial components of learned motion primitives as discrete vertices with connections to one another. 
This formulation allows us to bridge continuous control with discrete graph search. 
We propose two frameworks leveraging this graph: \emph{Stitching}, for constructing time-invariant DSs, and \emph{Chaining}, giving a sequence-based DS for complex motions while retaining convergence guarantees. 
Simulations and real-robot experiments show that these methods successfully generalize to new tasks where baseline methods fail.

%% file: sections/Introduction.tex
Learning from demonstration is a widely used paradigm in robotics for acquiring control policies from expert trajectories.
Recent work increasingly targets end-to-end foundation models that aim to generalize across tasks and embodiments~\cite{zitkovich2023rt,kim2024openvla}. 
However, such models typically require large datasets and significant compute, which can make fast, reactive control challenging~\cite{kawaharazuka2024real}.

An alternative is to learn a dynamical system (DS) from a small number of demonstrations and use it directly as a motion policy~\cite{billard2022learning,sun2024se}.
In \emph{Linear Parameter Varying DS} (LPV-DS)~\cite{billard2022learning}, a Gaussian Mixture Model (GMM) partitions the demonstrations into regions and smoothly blends locally linear time-invariant (LTI) DSs into a single nonlinear policy, 
enabling stable and convergent behavior toward a target state. 
Despite their data efficiency, LPV-DS-like methods are typically trained per task.
Consequently, they do not generalize to new initial-target pairs that were not demonstrated.

The ability to combine demonstrations of distinct tasks to solve novel, unseen tasks has high practical relevance in modern robotics.
In structured environments such as factories or laboratories, the number of potential tasks scales combinatorially with the number of start and goal configurations.
Collecting expert demonstrations for every possible variation is prohibitively expensive, time-consuming, and fundamentally unscalable. 
However, demonstrations in a shared workspace effectively exhibit \emph{regions of safe and kinematically viable movement} --- knowledge that we aim to repurpose in this work to achieve zero-shot generalization to new tasks.
This approach drastically improves the efficiency and scalability of imitation learning by circumventing the need for exhaustive data collection.
Furthermore, it naturally supports incremental learning; 
introducing a new workstation to a factory floor only requires a brief demonstration linking it to the existing network of learned behaviors. 

Prior work has explored the adjacent problem of adapting learned policies to new target configurations without requiring new data.
For instance, \emph{Elastic Motion Policy}~\cite{li2025elastic} geometrically reshapes the underlying GMM structure to reach new targets, 
while \emph{Task-Parameterized GMMs}~\cite{Calinon2016} warp a single demonstrated skill to accommodate new or moving reference frames.
Other approaches address multi-stage behaviors by segmenting demonstrations into sequential phases and recombining them temporally~\cite{medina2017learning}.
However, these methods primarily focus on adapting or segmenting an isolated task;
they do not treat disjoint demonstrations as a shared, interconnected map of safe, kinematically viable regions from which entirely novel paths can be routed. 

Building on LPV-DS, we reinterpret the Gaussians derived from the GMM as intermediate ``stepping stones'' that guide the system from its current state $\pos$ toward a target state $\pos^*$. 
By constructing a graph where these Gaussian centers serve as vertices, we can use established graph search algorithms to find paths that synthesize new motion policies from separate demonstrations, yielding efficient trajectories for unseen initial-target pairs in a zero-shot manner (see Fig.~\ref{fig:pipeline}).

\begin{figure*}[h]
    \centering
    \begin{subfigure}[t]{0.24\textwidth}
        \centering
        \includegraphics[width=\textwidth]{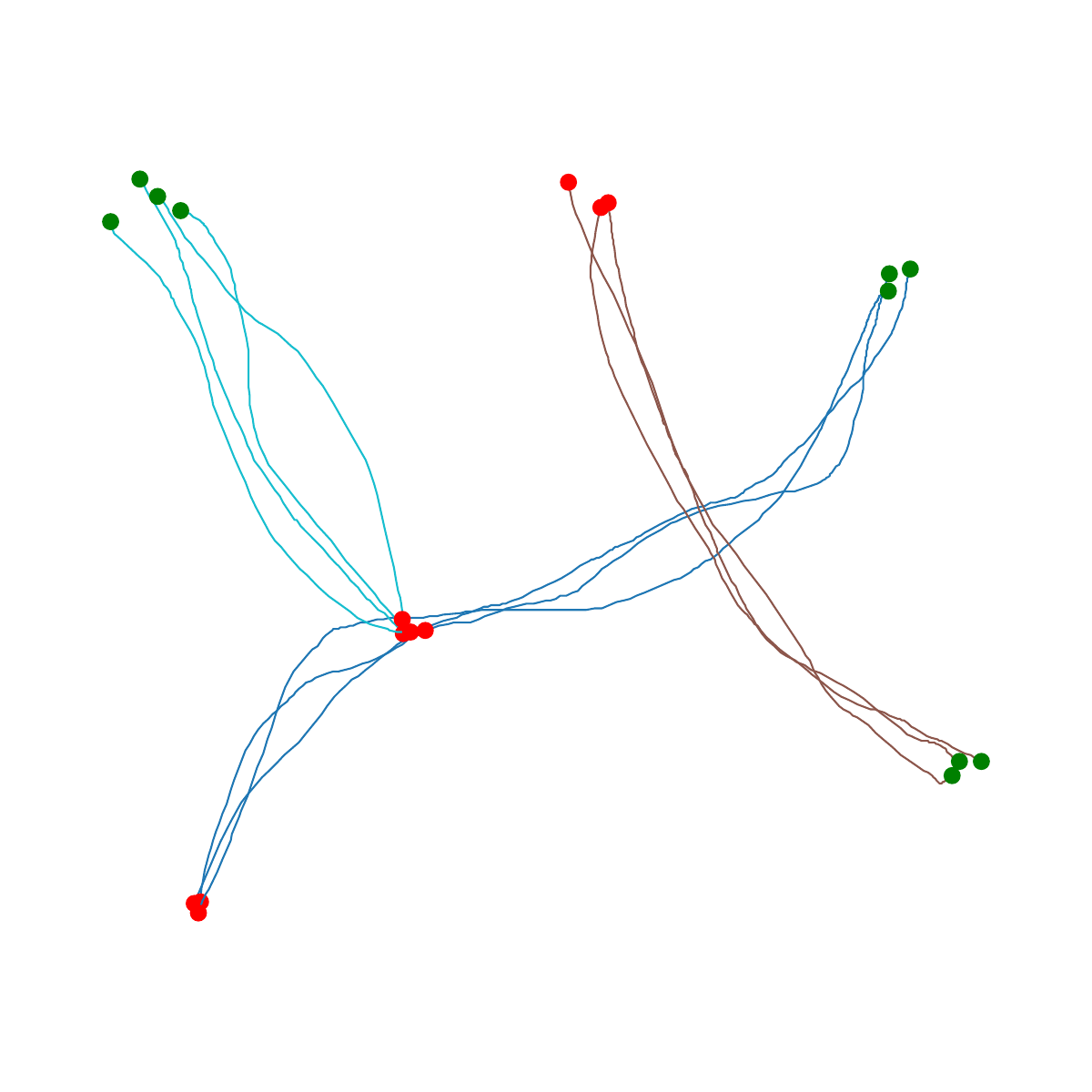}
        \caption{Set of Demonstrations.}
        \label{fig:pipeline:1}
    \end{subfigure}
    \hfill
    \begin{subfigure}[t]{0.24\textwidth}
        \centering
        \includegraphics[width=\textwidth]{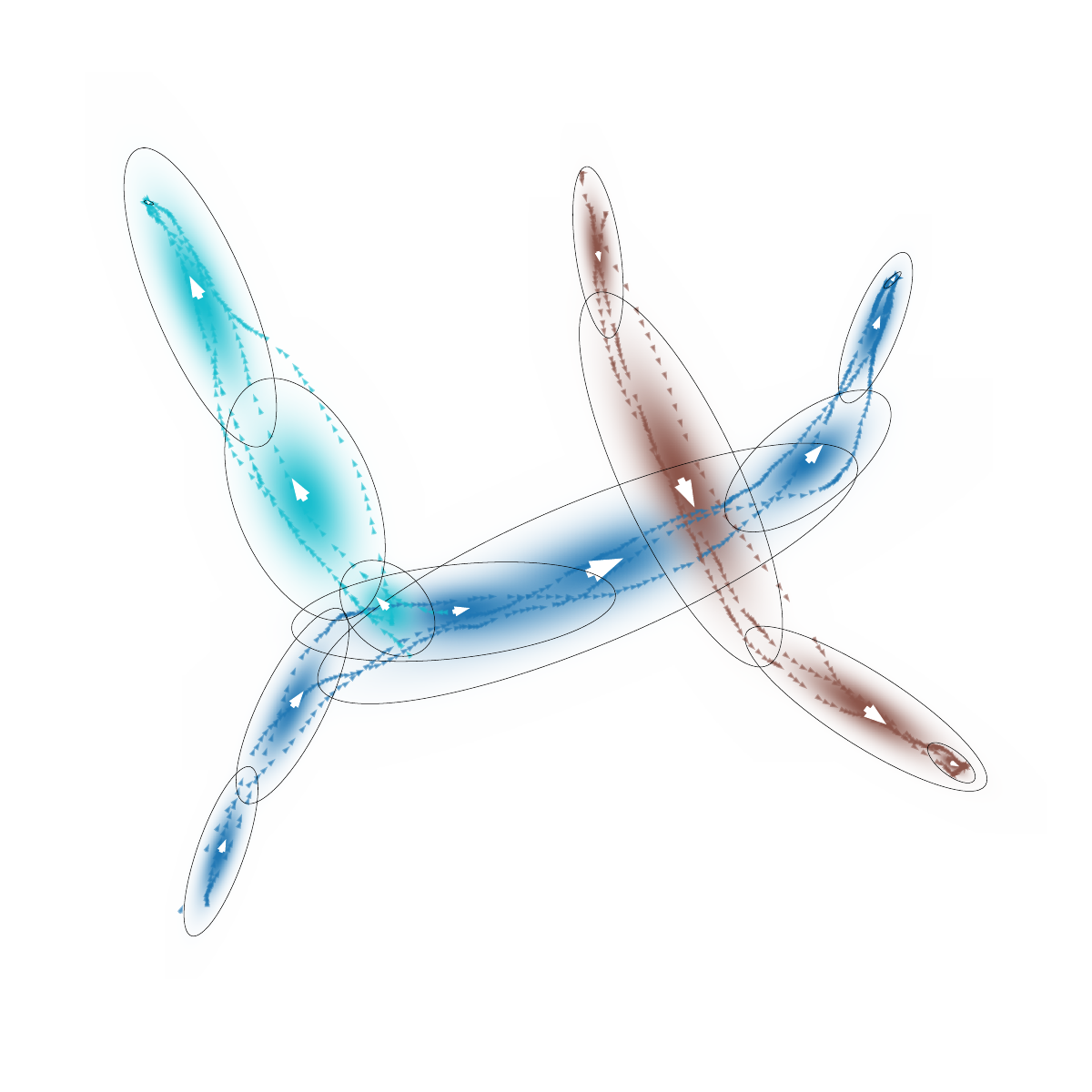}
        \caption{Demo-wise LPV-DS.}
        \label{fig:pipeline:2}
    \end{subfigure}
    \hfill
    \begin{subfigure}[t]{0.24\textwidth}
        \centering
        \includegraphics[width=\textwidth]{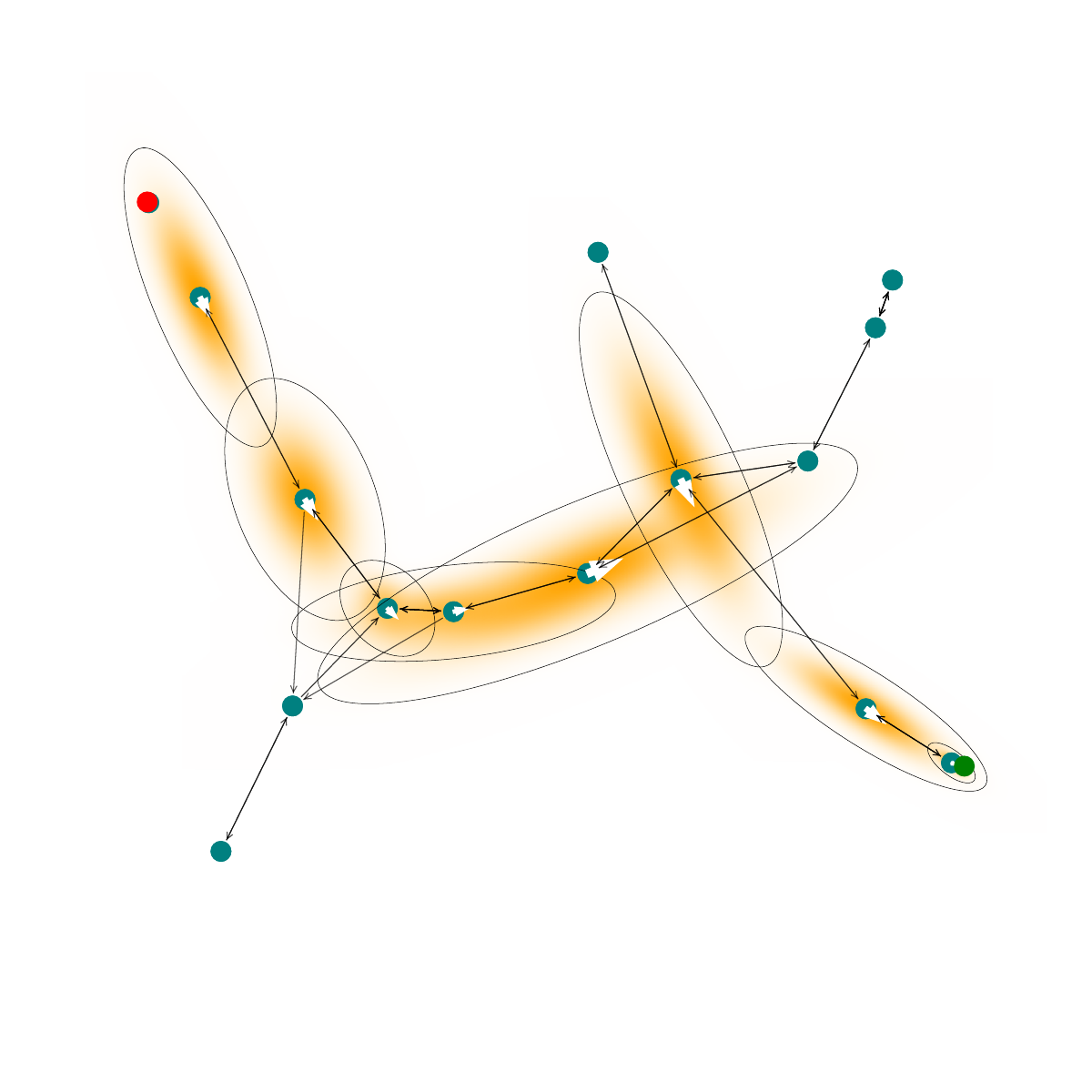}
        \caption{Gaussian Graph and solution.}
        \label{fig:pipeline:3}
    \end{subfigure}
    \hfill
    \begin{subfigure}[t]{0.24\textwidth}
        \centering
        \includegraphics[width=\textwidth]{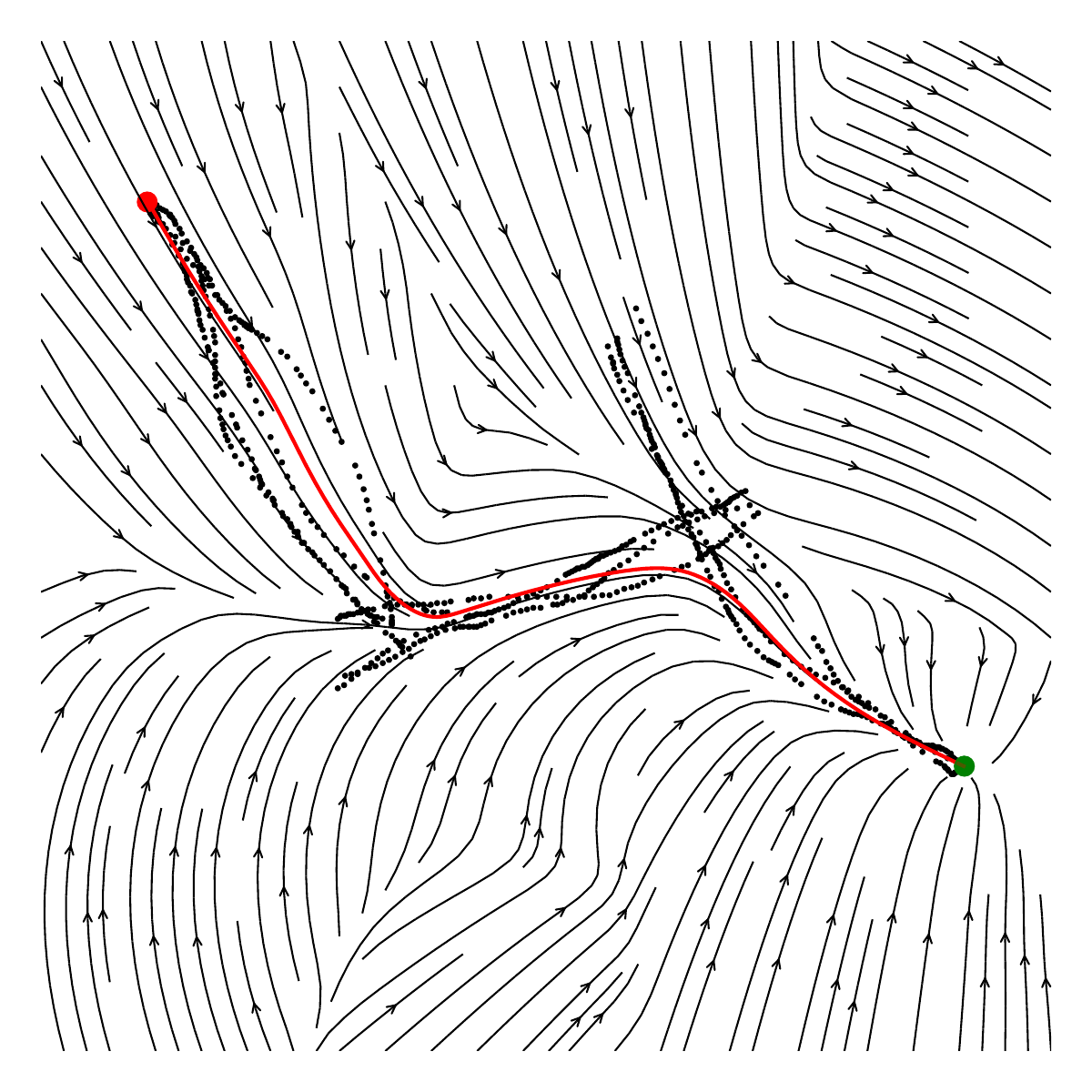}
        \caption{Stitched DS.}
        \label{fig:pipeline:4}
    \end{subfigure}
    \hfill
    \caption{The \emph{Demonstration Stitching} (Section~\ref{sec:DemoStitching}) pipeline using \emph{Shortest Path} and \emph{Reuse Gaussians}. On a set of separate demonstrations \textbf{(a)} (initial positions in red, target positions in green), LPV-DS is applied to each demonstration separately \textbf{(b)}, from which a Gaussian Graph is constructed and a shortest path from an initial position to a target position is found \textbf{(c)}. A DS is constructed to follow the shortest path \textbf{(d)}. Thus, a zero-shot generalization from separate motion demonstrations to a new, undemonstrated task is achieved.}
    \label{fig:pipeline}
    \vspace{-3mm}
\end{figure*}

Our main contributions can be summarized as:
\begin{itemize}
    \item The \emph{Gaussian Graph} (GG), a novel representation combining multiple demonstrations and enabling the use of well-known graph algorithms for synthesizing new motion policies.
    \item \emph{Demonstration Stitching}: GG-based construction of time-invariant DSs, enabling zero-shot initial-target generalization within data support and providing strict convergence guarantees.
    \item \emph{Demonstration Chaining}: A sequential GG-based DS producing sophisticated motions while maintaining the same generalization and convergence properties.
    \item An empirical evaluation demonstrating consistent generalization to unseen tasks and real-robot experiments verifying the method's capabilities.
\end{itemize}
Our implementation builds on LPV-DS~\cite{billard2022learning} and is provided along with datasets and additional details on our repository\footnote{\url{https://github.com/KilianFt/DemoStitching}\label{footnote:repo}.}. 

%% file: sections/Problem_Description.tex
We consider a setting where demonstrations for several different tasks are collected in the same workspace. Each demonstration $\demo = \{\pos_i^\mathit{ref}, \vel_i^\mathit{ref}\}_{i}$ --- with $\pos_i^\mathit{ref}, \vel_i^\mathit{ref}\in\Real^d$ --- aggregates position-velocity pairs from multiple reference trajectories. 
Intuitively, each demonstration provides position-velocity pairs that specify the desired robot motion within a region of the workspace (see Fig.~\ref{fig:pipeline:1}). Together, the different demonstrations form a demonstration set $\DemoSet$, which therefore contains a combined understanding of the workspace. The aim of this work is to effectively reuse information from these separate demonstrations for zero-shot generalization to unseen tasks by synthesizing new, safe movements from existing trajectories (see Fig.~\ref{fig:pipeline:4}) without collecting large amounts of data.


LPV-DS learns a time-invariant DS-based motion policy 
$\vel = f_\Params(\pos)$ 
from a single demonstration $\demo$.
The policy $f_\Params$ is parameterized by $\Params$ and constructed to ensure convergence to a stable equilibrium point $\pos^*\in\Real^d$.
That is,
$\lim_{t\rightarrow\infty} \| \pos - \pos^* \| = 0$ 
when starting from any position in $\Real^d$ and moving according to $\vel = f_\Params(\pos)$.
We say that $f_\Params$ is \emph{globally asymptotically stable} (GAS) with respect to an \emph{attractor} $\pos^*$.
This is done by formulating $f_\Params$ as a GMM of $K$ LTI dynamical systems:
\begin{equation}
    f_\Params(\pos) = \sum_{k=1}^{K} \gpost_k(\pos)\dsA_k(\pos - \pos^*),
\end{equation}
\begin{equation}
    \gpost_k(\pos) = \frac{\gprior_k \mathcal{N}(\pos\mid \gmean_k, \gcov_k)}{\sum_{j=1}^K \gprior_j \mathcal{N}(\pos\mid \gmean_j, \gcov_j)},
\end{equation}
where $\dsA_k \in \Real^{d\times d}$ is the state matrix describing the dynamics of the $k^\text{th}$ LTI DS component,
$\gprior_k$ is the prior probability of an observation belonging to Gaussian $k$,
and $\mathcal{N}(\pos\mid \gmean_k, \gcov_k)$ (with mean $\gmean_k$ and covariance $\gcov_k$) is the conditional probability density of observing $\pos$ at Gaussian $k$.
The posterior probability $\gpost_k(\pos)$ of $\pos$ belonging to system $k$ therefore weighs $\dsA_k$'s influence on $f_\Params$.
From this we see that the model parameters are $\Params = \left\{ \param_k^\mathit{DS}, \param_k^\mathit{GMM}\right\}_{k=1}^K$, with $\param_k^\mathit{DS} = \dsA_k$ and $\param_k^\mathit{GMM} = \left\{ \gprior_k, \gmean_k, \gcov_k\right\}$.
To find these parameters, first a GMM model is fit to $\demo$, giving $\param_k^\mathit{GMM}$ for $k=1,\dots,K$.
Then, for the DS parameters $\param_k^\mathit{DS}$, 
a parameterized Lyapunov function $V(\pos) = (\pos-\pos^*)^T \dsP (\pos-\pos^*)$ with $\dsP=\dsP^T \succ 0$ is formulated.
By ensuring that $\dsA_k^T \dsP + \dsP\dsA_k = \dsQ_k$ and $\dsQ_k = \dsQ_k^T \prec 0$
for all $k$, $f_\Params$ is GAS with respect to $\pos^*$.
$\dsP$ is found simultaneously with $\param_k^\mathit{DS}$.

By parameterizing $V(\pos)$ with $\dsP$, LPV-DS achieves greater flexibility in the LTI systems $\dsA_k$~\cite{figueroa2018physically} compared to methods like SEDS~\cite{medina2017learning}, which uses the unparameterized version.
However, LPV-DS --- just like SEDS but to a lesser degree --- remains constrained by $V(\pos)$ to guarantee GAS, and therefore struggles as the non-linearity of the underlying motions increases. 
Consequently, naively applying LPV-DS to aggregated demonstrations from multiple tasks causes the motion policy to fit the demonstration data less accurately as task diversity grows. 
Section~\ref{sec:chaining} proposes a method that circumvents this constraint while still guaranteeing GAS.



%% file: sections/GaussianGraph.tex
To aggregate multiple demonstrations, we introduce the \emph{Gaussian Graph}~(GG), which enables the construction of new motion policies for unseen tasks (see Fig.~\ref{fig:pipeline:3}).

First, LPV-DS is applied to each demonstration $\demo\in\DemoSet$ to obtain motion policies $\left\{ f_\demo \mid \demo\in\DemoSet \right\}$ (dropping the parameter $\Params$ subscript), where each $f_\demo$ has parameters $\Params_\demo$, $\dsP_\demo$, and $K_\demo$.
The GG is a directed, weighted graph $\Graph = \langle\Nodes,\Edges, \Weights\rangle$ with vertices $\Nodes$, edges $\Edges\subseteq\Nodes\times\Nodes$ and edge-weights $\Weights:\Edges\rightarrow\Real^+$.
Each vertex in $\Nodes$ corresponds to a component in the motion policies.
Specifically, for every component $k\in[1,..,K_\demo]$ of motion policy $f_\demo$ with parameters $\param_{k,\demo}^\mathit{GMM} = \left\{ \gprior_{k,\demo}, \gmean_{k,\demo}, \gcov_{k,\demo}\right\}$ and $\param_k^\mathit{DS} = \dsA_{k,\demo}$, there exists a vertex $v_{k,\demo}\in\Nodes$.
The \emph{direction} $\direction_{k,\demo}$ is defined by the dynamics at the Gaussian center: 
$\direction_{k,\demo} = \dsA_{k,\demo}(\gmean_{k,\demo} - \pos^*_\demo)$, 
where $\pos_\demo^*$ is the attractor of $\demo$.

Informally, we assume the full system dynamics $f_\demo$ near component $k$ strongly aligns with $k$'s direction,
$f_\demo(\gmean_{k,\demo}) \approx \direction_{k,\demo}$.
Under this assumption, dynamics near $\gmean_{k,\demo}$ lead toward components in the direction of $\direction_{k,\demo}$, 
with closer components exerting greater influence until their own dynamics eventually dominate.
Based on this principle, the connection from vertex $v_i$ to $v_j$ is stronger when they are closer and $v_j$ aligns more with $\direction_i$.

Edges in $\Edges$ represent possible component transitions.
An edge $(v_i, v_j)\in\Edges$ exists only if:
\begin{itemize}
    \item $\cos(\varphi_{ij}) > 0$, where $\varphi_{ij}$ is the angle between $\direction_i$ and $\gmean_j - \gmean_i$, ensuring $\direction_{i}$ points at least partly toward component $j$, and
    \item the Bhattacharyya Coefficient~\cite{bhattacharyya1943measure} between Gaussians of $i$ and $j$ exceeds the user-defined threshold $\parambhat$, ensuring sufficient data connectivity between components.
\end{itemize}
The edge weight is defined as
\begin{equation}
    \label{eq:graph:weight}
    \Weights((v_i,v_j)) = \frac{||\gmean_{i}-\gmean_j||^\paramdist}{\cos(\varphi_{i,j})^\paramdir},
\end{equation}
where $\paramdist$ and $\paramdir$ are user-defined parameters. 

The user-defined parameters ($\parambhat, \paramdist, \paramdir$) collectively govern the GG's connectivity and traversal cost, balancing adherence to demonstrations with the flexibility to synthesize novel behaviors. 
The threshold $\parambhat$ acts as a structural safety filter by dictating the minimum data support required for a transition. Higher $\parambhat$ values strictly enforce empirical support, while lower values suit applications more tolerant of out-of-data movements;
however, excessively high thresholds risk fragmenting the GG.
For viable edges, $\paramdist$ and $\paramdir$ act as cost regularizers on spatial separation and dynamic consistency, respectively. 
A high $\paramdist$ heavily penalizes long jumps, favoring paths of closely packed Gaussians, while a high $\paramdir$ enforces strict directional alignment between a Gaussian's inherent dynamics and the physical path to the next vertex. Conversely, lowering these two parameters widens the transition ``field-of-view'', enabling the graph to bridge disjoint tasks at the risk of leaving data-supported regions.

Depending on the application, some demonstrations may be safely executed in reverse. 
We refer to this as the \emph{bi-directionality assumption}.
The graph is then expanded by creating a counterpart vertex $v_i'$ for every original vertex $v_i$ corresponding to such a demonstration, with reversed dynamics $\dsA_i' = -\dsA_i$ and direction $\direction_i' =- \direction_i$.

Finally, the methods in subsequent sections will primarily use shortest paths in the GG. 
For this reason, the GG can be reduced in size by removing every edge $(v_i, v_j)\in\Edges$ not included in the shortest path from $v_i$ to $v_j$. 
This operation does not have any effect on shortest paths in the GG since such edges would never be included, however, it does improve computation times.

%% file: sections/TimeInvariantStitching.tex
\label{sec:DemoStitching}

In this section, we present four methods for stitching together separate demonstrations into a single time-invariant motion policy.
These four methods arise from the combination of two ways to use the GG (\emph{Shortest Path} and \emph{Shortest Path Tree}) with two ways of reusing elements from the initial LPV-DS policies $\left\{ f_\demo \mid \demo\in\DemoSet \right\}$.

\subsection{Graph Strategies}

\subsubsection{Shortest Path}
\label{sec:shortest_path}

We assume for this method that an initial position $\pos_0$ and attractor $\pos^*$ are given.
A vertex is added to the GG for each $v_0$ and $v^*$ respectively, with $\gmean_0=\pos_0$ and $\gmean^*=\pos^*$.
For $v^*$, only incoming edges connect it to existing vertices; for every $v_j\in\Nodes$,~\eqref{eq:graph:weight} is used to compute an initial weight $w_{j,*}$, with the same requirement that $\cos(\varphi_{j,*})>0$. 
However, the weight is further modified to account for data support provided by $v_j$, giving $\Weights((v_j,v^*)) = w_{j,*} / \mathcal{N}(\pos_0 \mid \gmean_j, \gcov_j)$.
On the other hand, only outgoing edges from $v_0$ are added. 
Since $v_0$ has no notion of direction, for every $v_j\in\Nodes$, $v_0$ adopts direction $\direction_j$ when calculating $w_{0,j}$ with~\eqref{eq:graph:weight}, also requiring $\cos(\varphi_{0,j})>0$. 
Factoring in the amount of data-support provided by $v_j$, we get $\Weights((v_0, v_j)) = w_{0,j} / \mathcal{N}(\pos_0 \mid \gmean_j, \gcov_j)$.
A standard shortest-path algorithm (e.g., Dijkstra's~\cite{dijkstra}) is used to find the shortest path $\langle v_0, v_1, \dots,v_e, v^* \rangle$ from $v_0$ to $v^*$,
with the set $\ggSolution = \left\{v_1, \dots, v_e\right\}\subseteq\Nodes$ containing the vertices corresponding to components.

\subsubsection{Shortest Path Tree}\label{sec:all_paths}
Given only an attractor $\pos^*$, a corresponding vertex $v^*$ is added to the GG in the same way as above. The shortest path tree to $v^*$ is found, with all vertices in $\Nodes$ for which there exists a path to $v^*$ stored in $\ggSolution$.
This approach yields a global policy that is agnostic to a starting position.
The bi-directionality assumption, however, introduces a critical problem:
a path to $v^*$ may exist from both a vertex $v_i$ and its reverse counterpart $v_i'$. 
Including both vertices in the final policy is problematic, as their opposing dynamics ($\direction_i = -\direction_{i}'$) effectively cancel each other out.
To resolve this ambiguity, we prune the vertex set by applying a simple heuristic: for any such pair $\langle v_i, v_i' \rangle$, we retain only the vertex with the shorter path to $v^*$. 


\subsection{Levels of DS Reuse}\label{sec:stitch}


Using the set of vertices $\ggSolution$ from either of the two graph strategies introduced above, a new motion policy can be constructed.
Intuitively, reusing components from the initial policies $\left\{ f_\demo \mid \demo\in\DemoSet \right\}$ simplifies computation --- but also constrains the new policy.
We consider two levels of reuse:

\subsubsection{No Reuse}

When LPV-DS is applied to a demonstration $\demo\in\DemoSet$, data points in $\demo$ are clustered to each of the $K_\demo$ Gaussian components, resulting in the clustering $C_{k,\demo}\subseteq\demo$ for each $k=1,\dots,K_\demo$.
In this method, all data points clustered to vertices in $\ggSolution$ are collected, $\bigcup_{n_{i}\in \ggSolution} C_{i}$, and a completely new LPV-DS model is fitted.
Essentially, the GG is used to filter data points from the demonstrations, after which LPV-DS is applied anew.
With the bi-directionality assumption, the data points' velocity $\vel$ associated with a reversed vertex are negated (that is, $-\vel$ is used).

\subsubsection{Reusing Gaussians}
This method reuses the GMM parameters $\param_{i}^\mathit{GMM} = \left\{ \gprior_{i}, \gmean_{i}, \gcov_{i}\right\}$ and corresponding data points $C_i$ for each vertex $v_i\in\ggSolution$.
The priors $\{\gprior_i\}$, however, are normalized to sum to $1$.
With the GMM structure fixed, only the second LPV-DS step is performed to find the DS parameters $\param_i^{DS} = \dsA_i$ and a shared Lyapunov parameter $\dsP$ that best fits the filtered data points.


We explored not only reusing Gaussians but also $\dsP$ to only recompute the dynamics $\dsA_i$ or reusing $\dsA_i$, and instead seek a single $\dsP$. However, those approaches yielded relatively low performance and are therefore omitted.


%% file: sections/Chaining.tex
The overall approach of \emph{Demonstration Chaining} is to transition through a sequence of simpler DSs to collectively realize motions more sophisticated than those achievable with any single time-invariant DS alone.
For instance, the motions shown in Fig.~\ref{fig:chaining_example} are impossible for standard LPV-DS or the stitching methods presented in the previous section: 
they require multiple velocities at a single position --- violating time-invariance --- and follow paths incompatible with the decay of any parameterized Lyapunov function.
In this section, we 
define a \emph{DS-Chain},
formulate criteria under which it is GAS with respect to an attractor $\pos^*$,
and then describe our specific method for using the GG to construct such a GAS DS-Chain.

\begin{figure}[t]
    \centering
    \begin{subfigure}[t]{0.48\columnwidth}
        \centering
        \includegraphics[width=\textwidth]{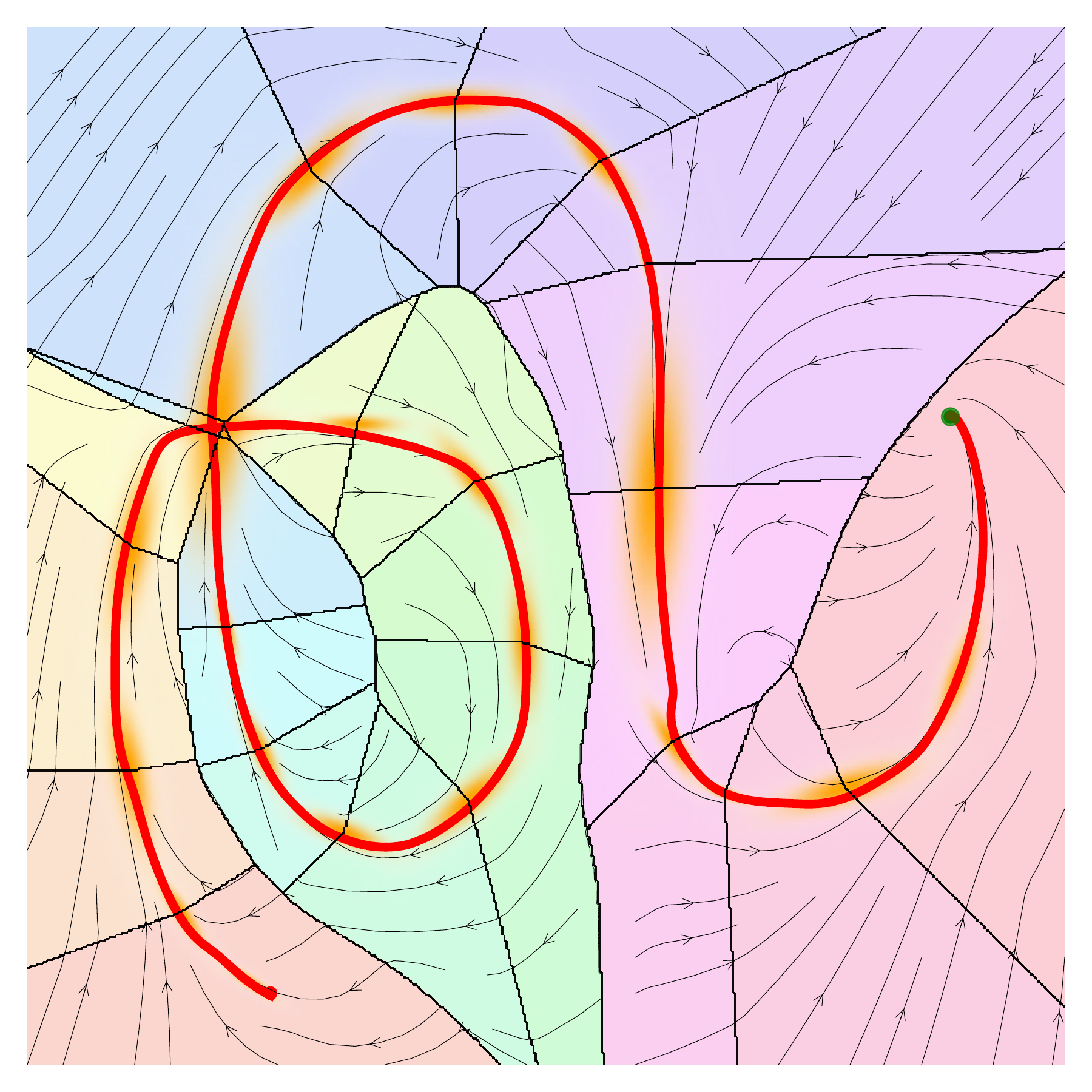}
        \caption{}
        \label{fig:chaining_example:1}
    \end{subfigure}
    \hfill
    \begin{subfigure}[t]{0.48\columnwidth}
        \centering
        \includegraphics[width=\columnwidth]{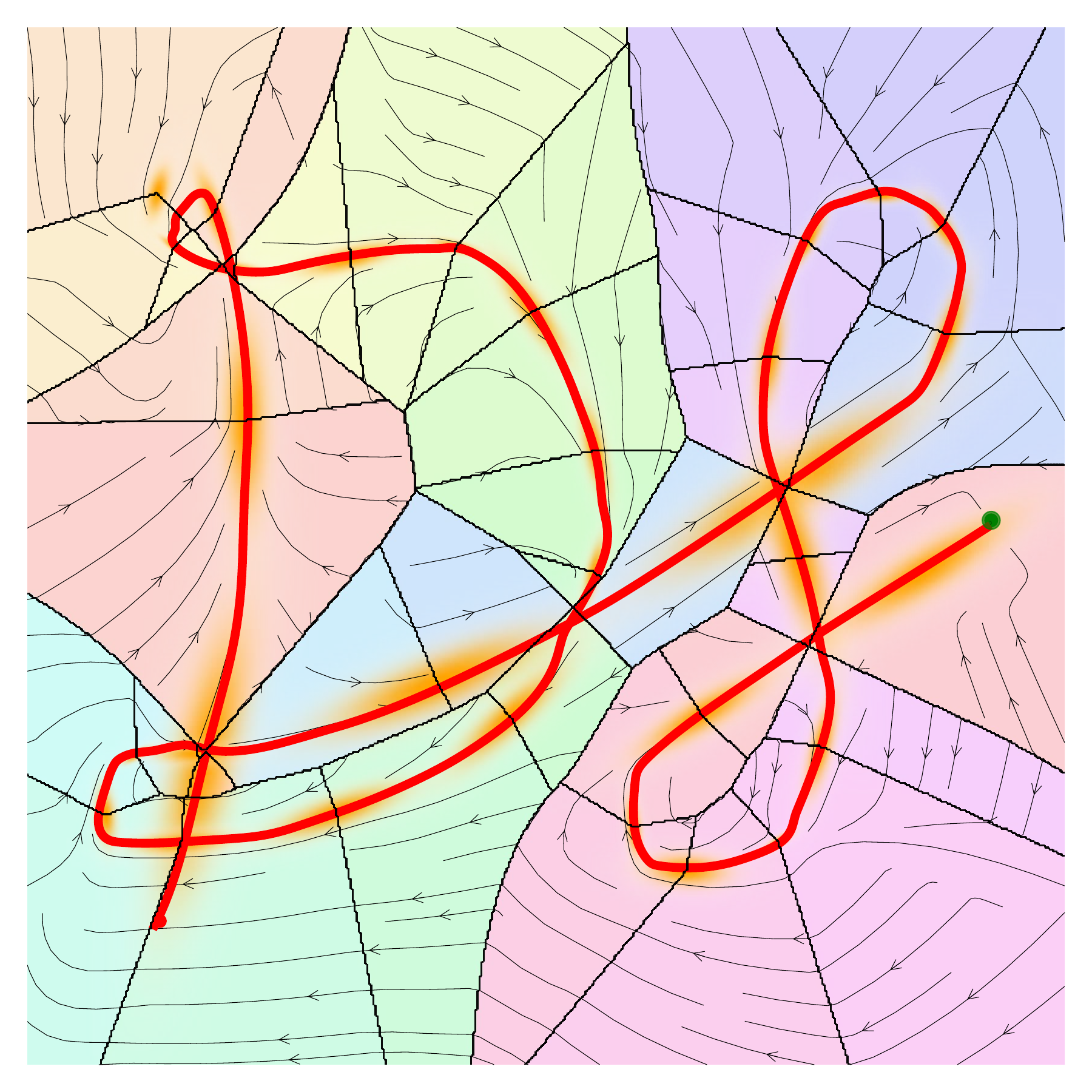}
        \caption{}
        \label{fig:chaining_example:2}
    \end{subfigure}
    \caption{Examples of two DS-Chains, with the resulting motions when starting at an initial position in red. Colored regions show the active DS at the corresponding part of the trajectory, with the underlying Gaussians shown in orange.}
    \label{fig:chaining_example}
    \vspace{-3mm}
\end{figure}

\begin{definition}[DS-Chain]
    A DS-Chain is a tuple $\DSC = \langle \DSseq, \TriggerSeq, \TransSeq, \InterDSseq \rangle$, where 
    $\DSseq = \langle f_1, \dots, f_N \rangle$ contains $N$ dynamical systems;
    $\TriggerSeq = \langle \trigger_1, \dots, \trigger_{N-1} \rangle$ and $\TransSeq = \langle \trans_1, \dots, \trans_{N-1} \rangle$ contain $N-1$ \emph{trigger} and \emph{timer} predicates, respectively, each mapping to $\{ \top, \perp\}$; and
    $\InterDSseq = \langle \inter_1, \dots, \inter_{N-1} \rangle$ contains $N-1$ dynamical systems.
\end{definition}

We model a DS-Chain $\DSC$ as a finite hybrid automaton with \emph{nominal states} $S = \left\{s_1,\dots, s_{N}\right\}$ and \emph{intermediate states} $S' = \left\{s_1', \dots, s_{N-1}'\right\}$.
The system initializes in $s_1$. 
Each nominal state $s_i$ transitions to $s_i'$ when the trigger $\gamma_i$ is $\top$, and each intermediate state $s_i'$ transitions to $s_{i+1}$ when the timer $\trans_i$ is $\top$.
Thus, the automaton's structure is $s_1 \overset{\gamma_1}{\rightarrow} s_1' \overset{\trans_1}{\rightarrow} s_2 \overset{\gamma_2}{\rightarrow} s_2' \overset{\trans_2}{\rightarrow} \dots \overset{\trans_{N-1}}{\rightarrow} s_N$.
The system dynamics are defined by the state of the DS-Chain: 
when in state $s_i$ the dynamics are governed by $f_i$, when in state $s_i'$ the dynamics are governed by $\inter_i$.
Fig.~\ref{fig:chaining_example} visualizes this, with colored regions corresponding to the active nominal state in $S$ and associated DS in $\DSseq$ during that part of the trajectory. Note that the transition DSs in $\InterDSseq$ are not visualized. 
Theorem~\ref{theorem:DSC_GAS} provides criteria under which $\DSC$ is GAS with respect to $\pos^*$.

\begin{theorem}[Global Asymptotically Stable DS-Chain]
    \label{theorem:DSC_GAS}
    A DS-Chain $\DSC = \langle \DSseq, \TriggerSeq, \TransSeq, \InterDSseq \rangle$ --- modeled by states in $S$ and $S'$ --- is GAS with respect to an attractor $\pos^*$ if:
    \begin{enumerate}
        \item the DSs in $\DSseq$ and $\InterDSseq$ are bounded, \label{theorem:DSC_GAS:criteria_1}
        \item $\forall i = 1,\dots,N-1$, trigger $\trigger_i$ is guaranteed to evaluate to $\top$ within a finite time when in state $s_i$, \label{theorem:DSC_GAS:criteria_2}
        \item $\forall i = 1,\dots,N-1$, $\trans_i$ is guaranteed to evaluate to $\top$ within a finite time when in state $s_i'$, \label{theorem:DSC_GAS:criteria_3}
        \item $f_N$ is GAS with respect to $\pos^*$. \label{theorem:DSC_GAS:criteria_4}
    \end{enumerate}
\end{theorem}
\begin{proof}
    Since $\trigger_i$ is guaranteed to evaluate to $\top$ within a finite time when in state $s_i$, it holds that the transition from $s_i$ to $s_i'$ will occur within a finite time.   
    Likewise, since $\trans_i$ is guaranteed to evaluate to $\top$ within a finite time when in state $s_i'$, it holds that the transition from $s_i'$ to $s_{i+1}$ will occur within a finite time.
    Thus, the system initializes in state $s_1$ and is guaranteed to transition through all states to eventually arrive and remain in $s_N$.
    Since all dynamics $f_1, \dots, f_N$ and $\inter_1, \dots, \inter_{N-1}$ are bounded, the system's state will be bounded when entering $s_N$, from which the dynamics are governed by $f_N$ --- with $f_N$ being GAS with respect to $\pos^*$.
    In other words, starting from any bounded position and following the dynamics specified by $\DSC$, eventually the system's state will be bounded and the dynamics will be governed by $f_N$, inevitably leading to $\pos^*$.
    Therefore, $\DSC$ is GAS with respect to $\pos^*$.
\end{proof}

\begin{figure*}[t]
    \centering
    \begin{subfigure}[t]{0.193\textwidth}
        \centering
        \includegraphics[width=\textwidth]{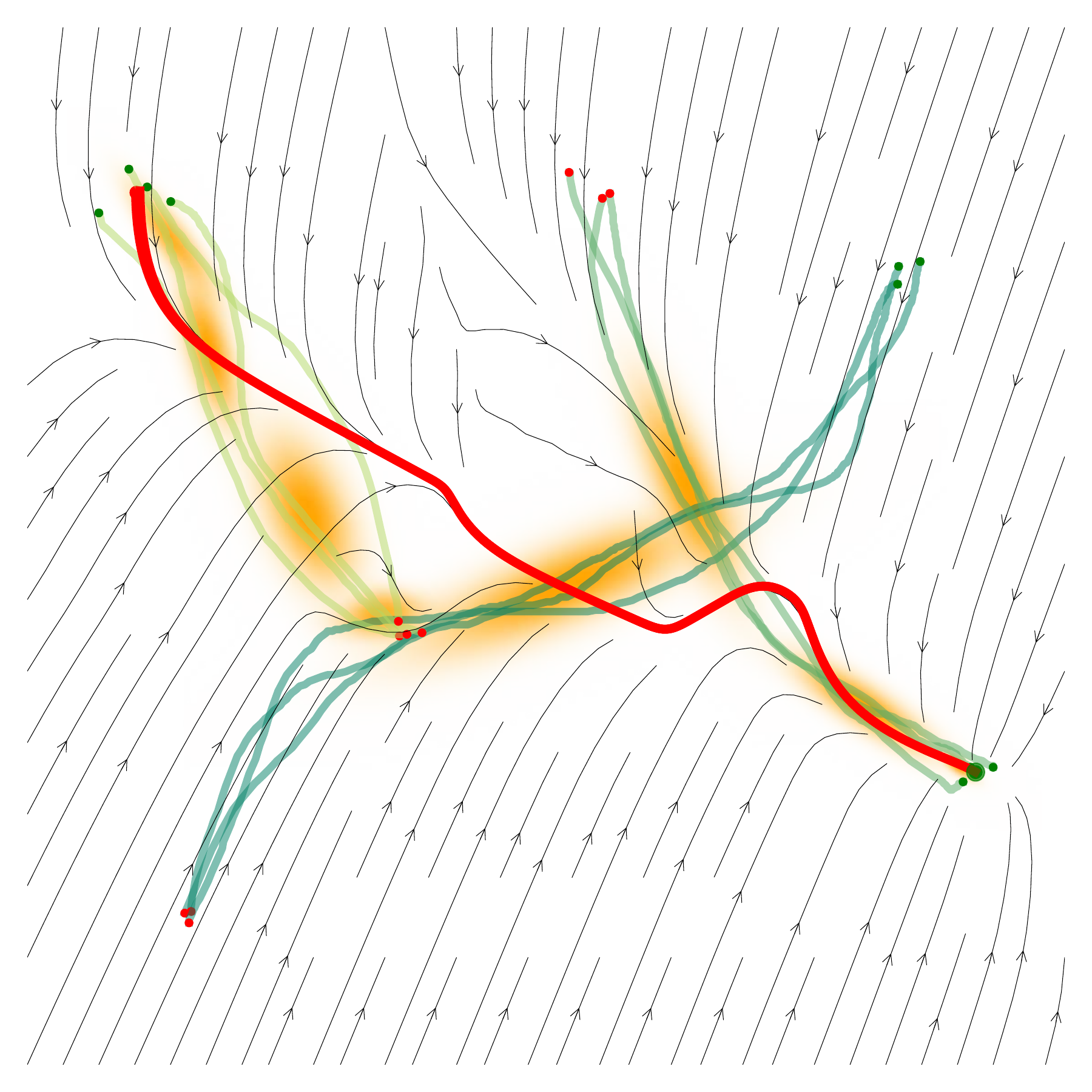}
        \caption{Stitch-SP (All)}
        \label{fig:methods_comparison_a}
    \end{subfigure}
    \hfill
    \begin{subfigure}[t]{0.193\textwidth}
        \centering
        \includegraphics[width=\textwidth]{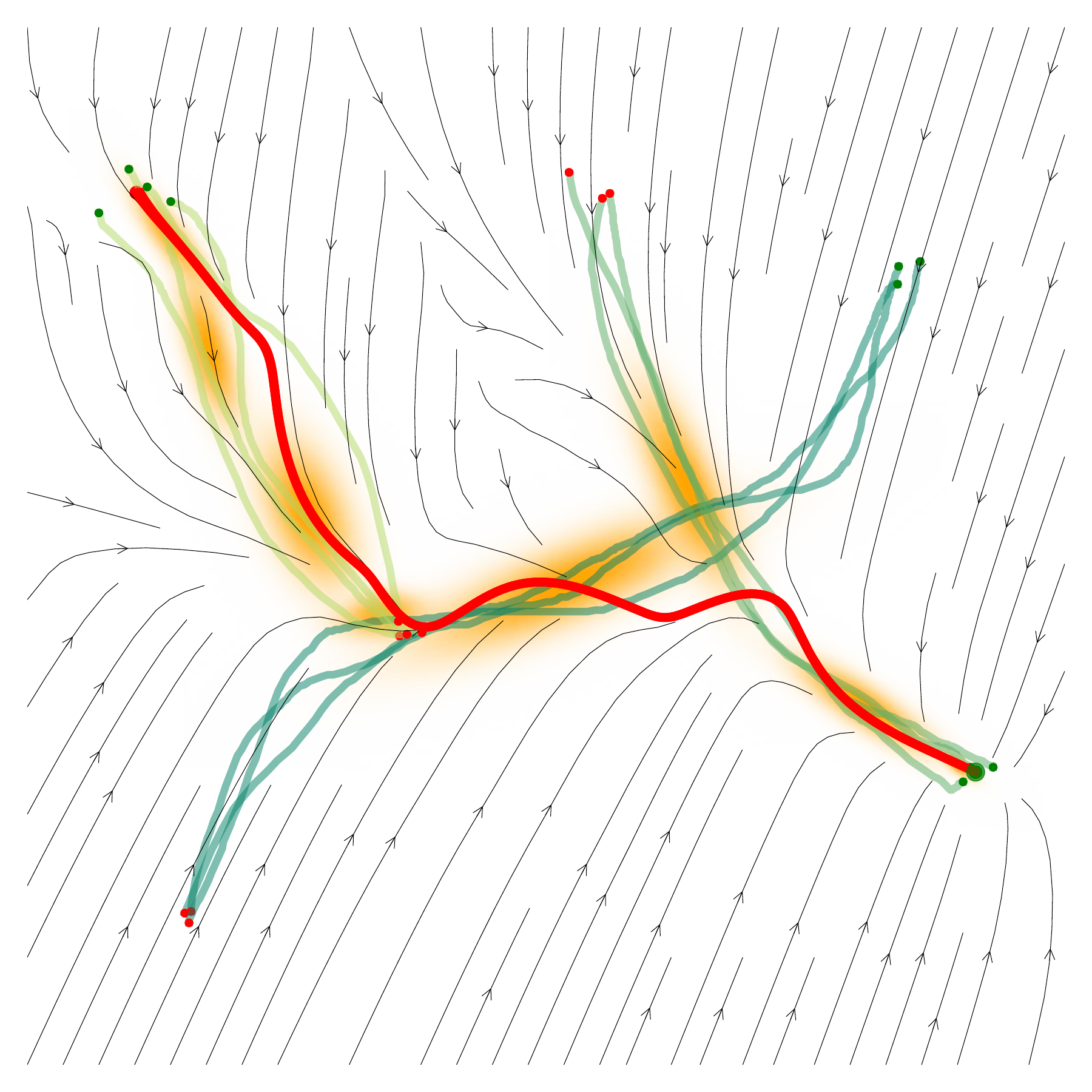}
        \caption{Stitch-SP (DS)}
        \label{fig:}
    \end{subfigure}
    \hfill
    \begin{subfigure}[t]{0.193\textwidth}
        \centering
        \includegraphics[width=\textwidth]{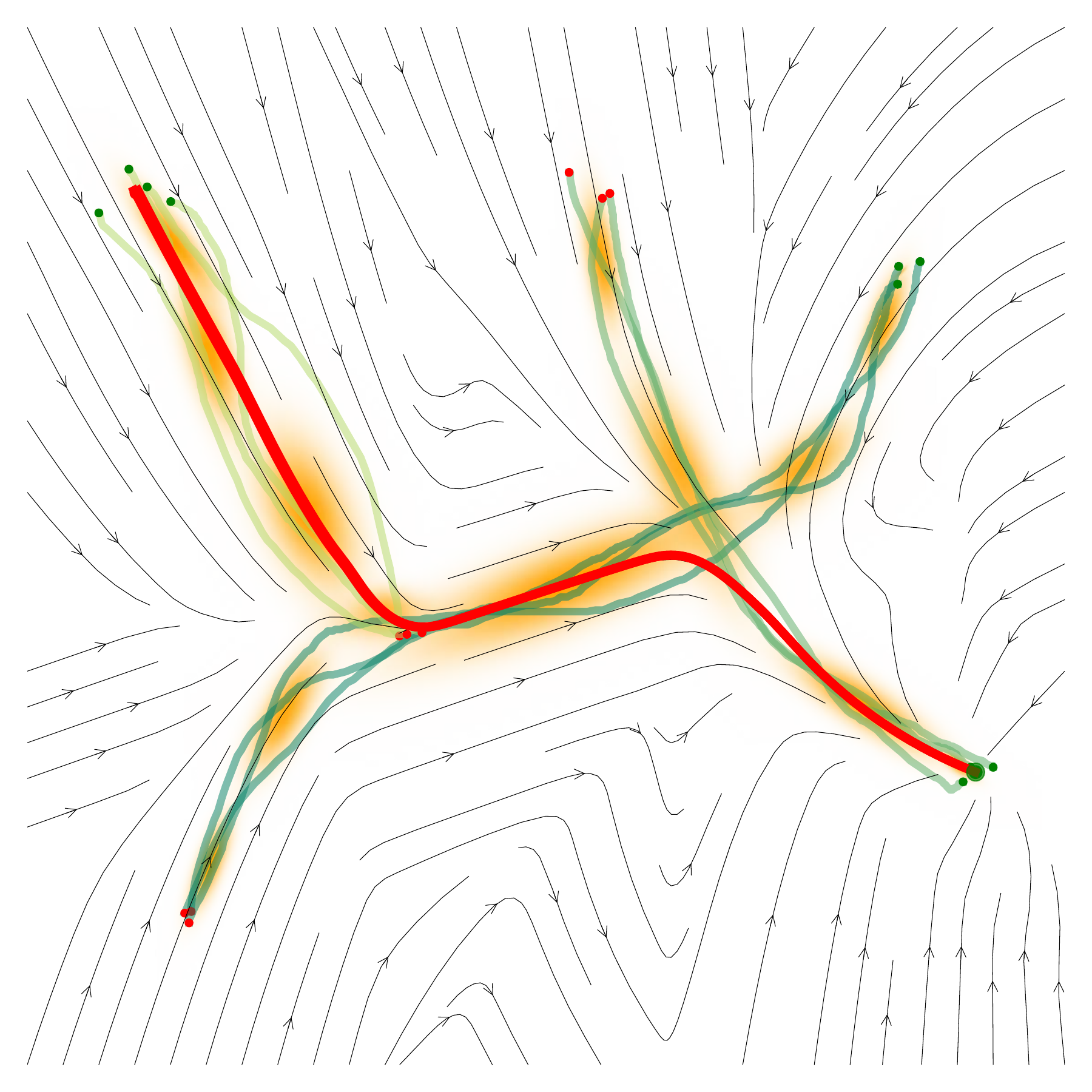}
        \caption{Stitch-SPT (All)}
        \label{fig:}
    \end{subfigure}
    \hfill
    \begin{subfigure}[t]{0.193\textwidth}
        \centering
        \includegraphics[width=\textwidth]{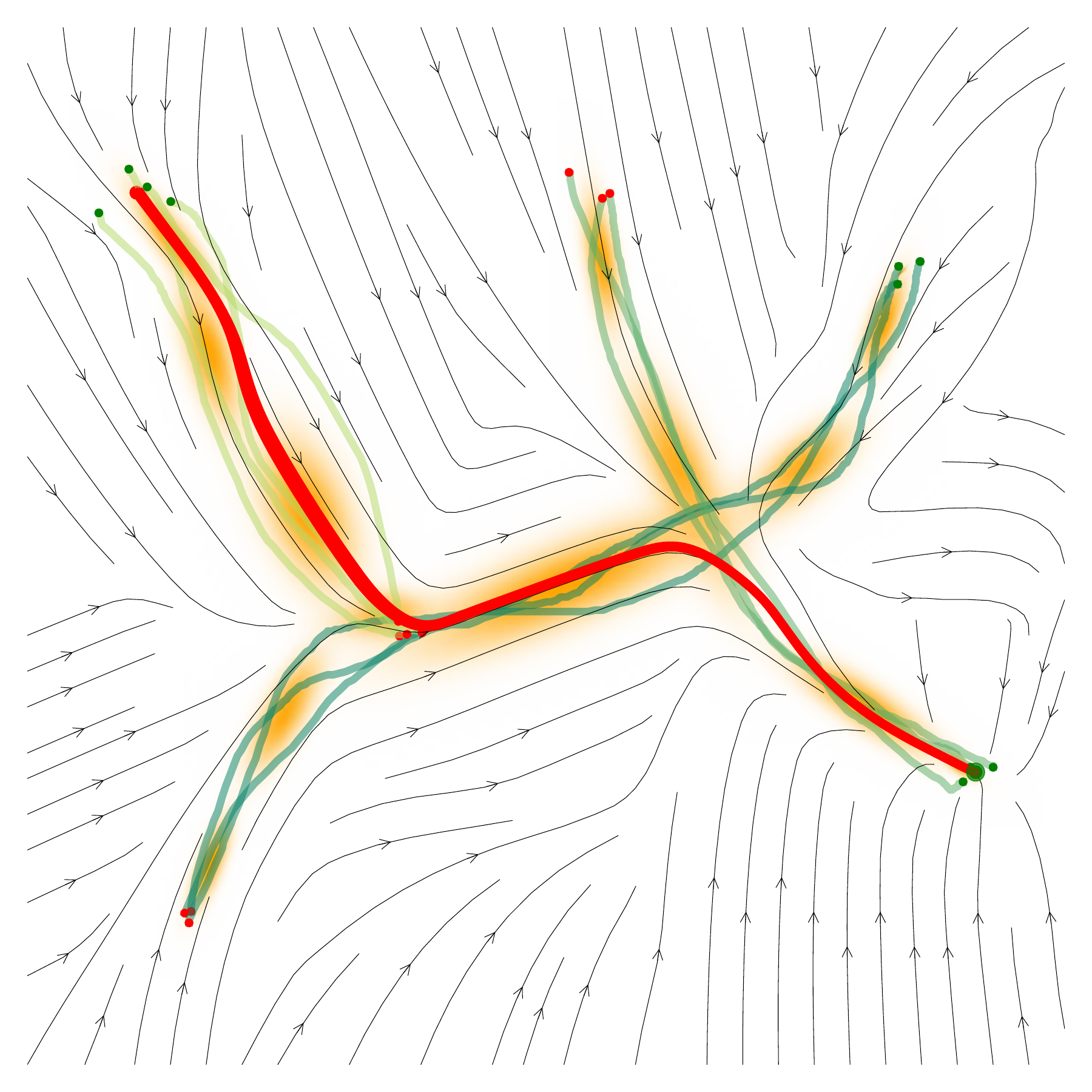}
        \caption{Stitch-SPT (DS)}
        \label{fig:}
    \end{subfigure}
    \hfill
    \begin{subfigure}[t]{0.193\textwidth}
        \centering
        \includegraphics[width=\textwidth]{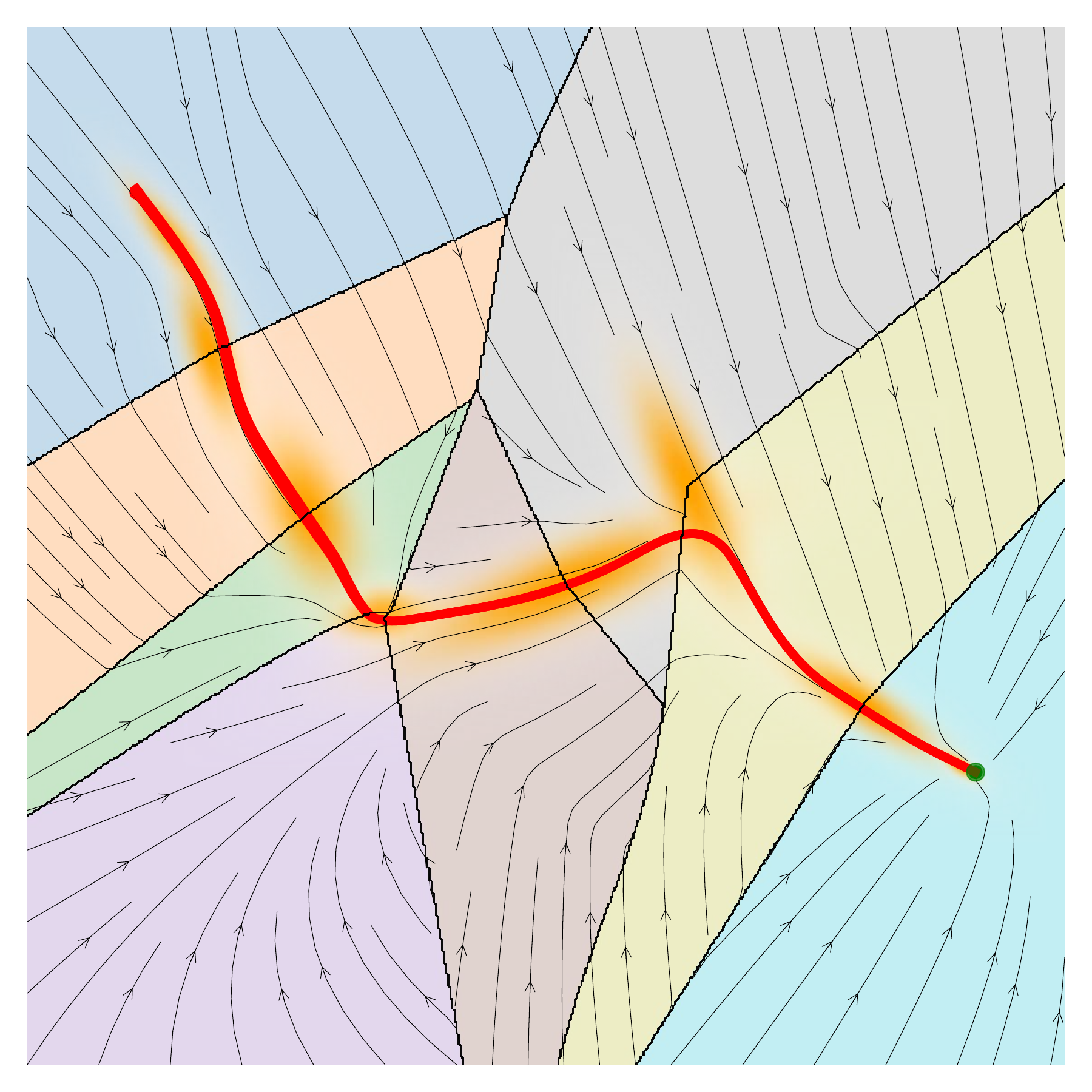}
        \caption{Chaining (All)}
        \label{fig:methods_comparison_e}
    \end{subfigure}
    \begin{subfigure}[t]{0.193\textwidth}
        \centering
        \includegraphics[width=\textwidth]{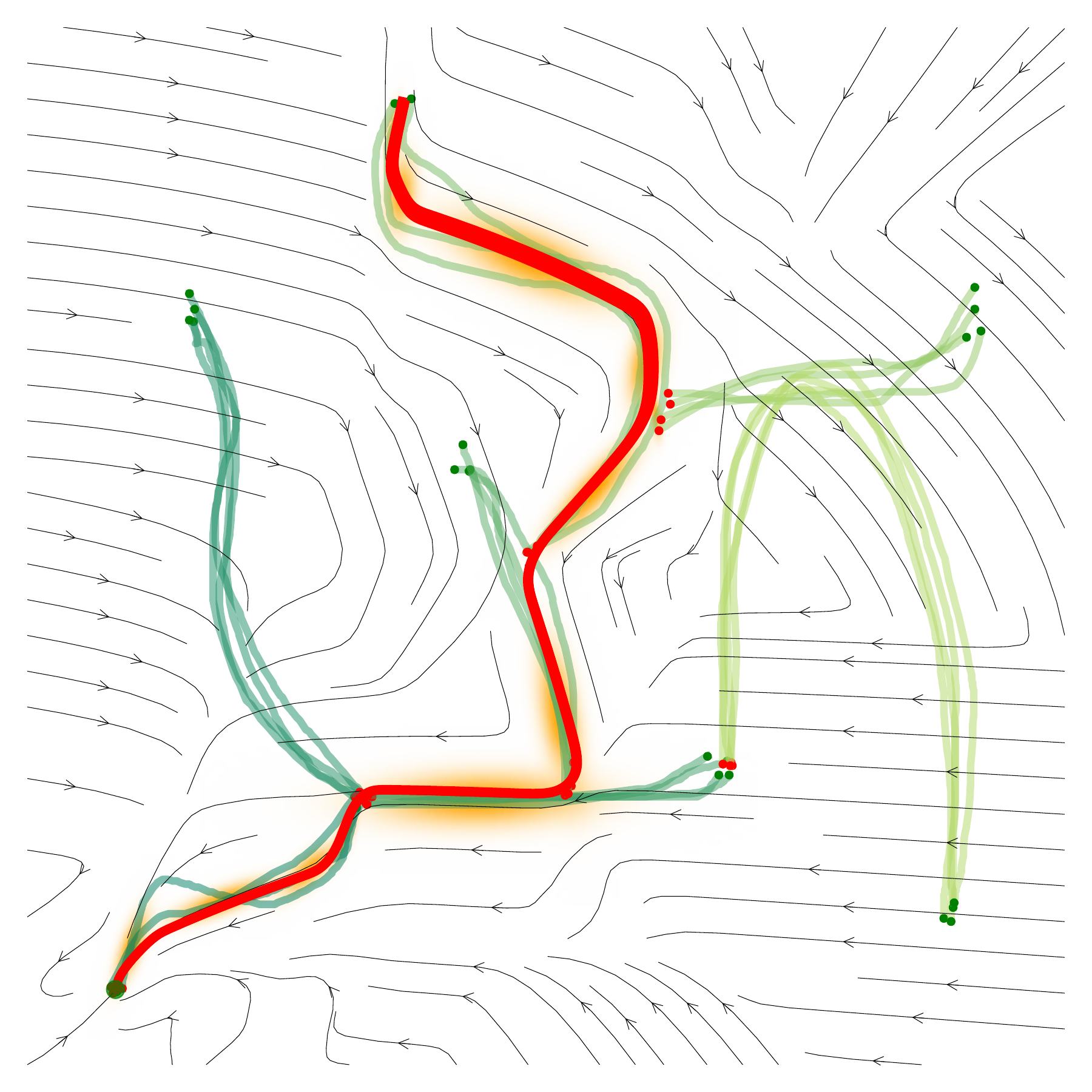}
        \caption{Stitch-SP (All)}
        \label{fig:methods_comparison_f}
    \end{subfigure}
    \hfill
    \begin{subfigure}[t]{0.193\textwidth}
        \centering
        \includegraphics[width=\textwidth]{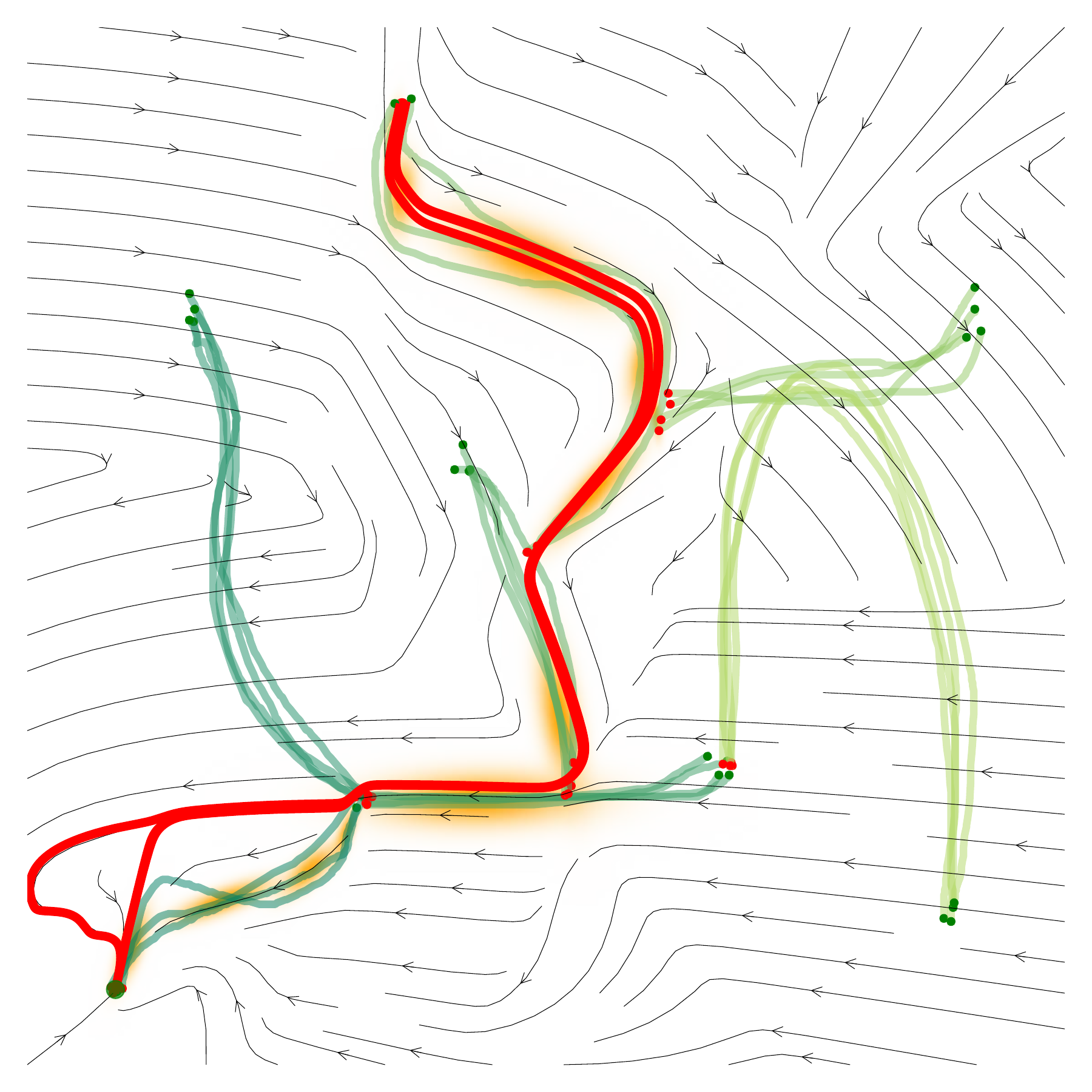}
        \caption{Stitch-SP (DS)}
        \label{fig:}
    \end{subfigure}
    \hfill
    \begin{subfigure}[t]{0.193\textwidth}
        \centering
        \includegraphics[width=\textwidth]{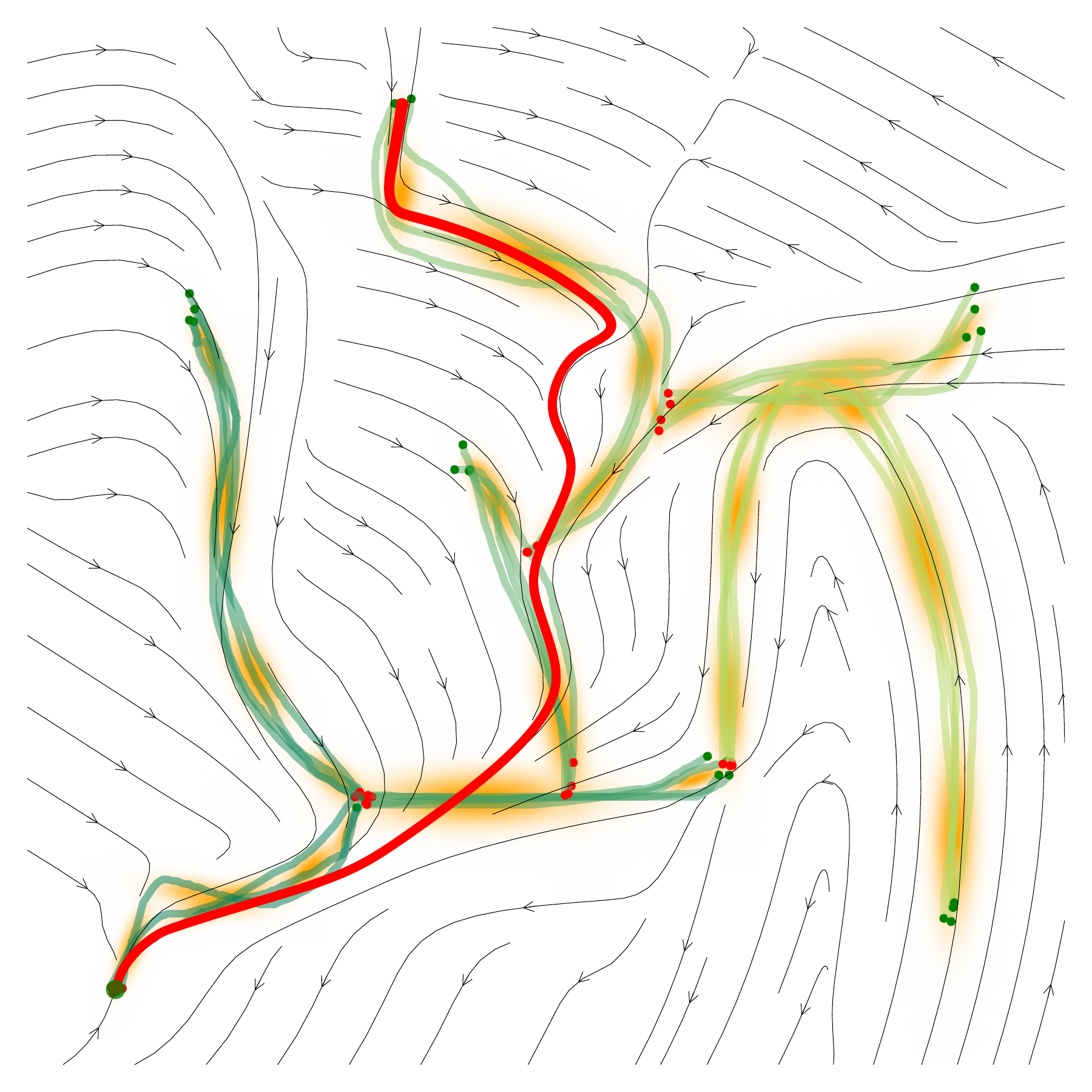}
        \caption{Stitch-SPT (All)}
        \label{fig:}
    \end{subfigure}
    \hfill
    \begin{subfigure}[t]{0.193\textwidth}
        \centering
        \includegraphics[width=\textwidth]{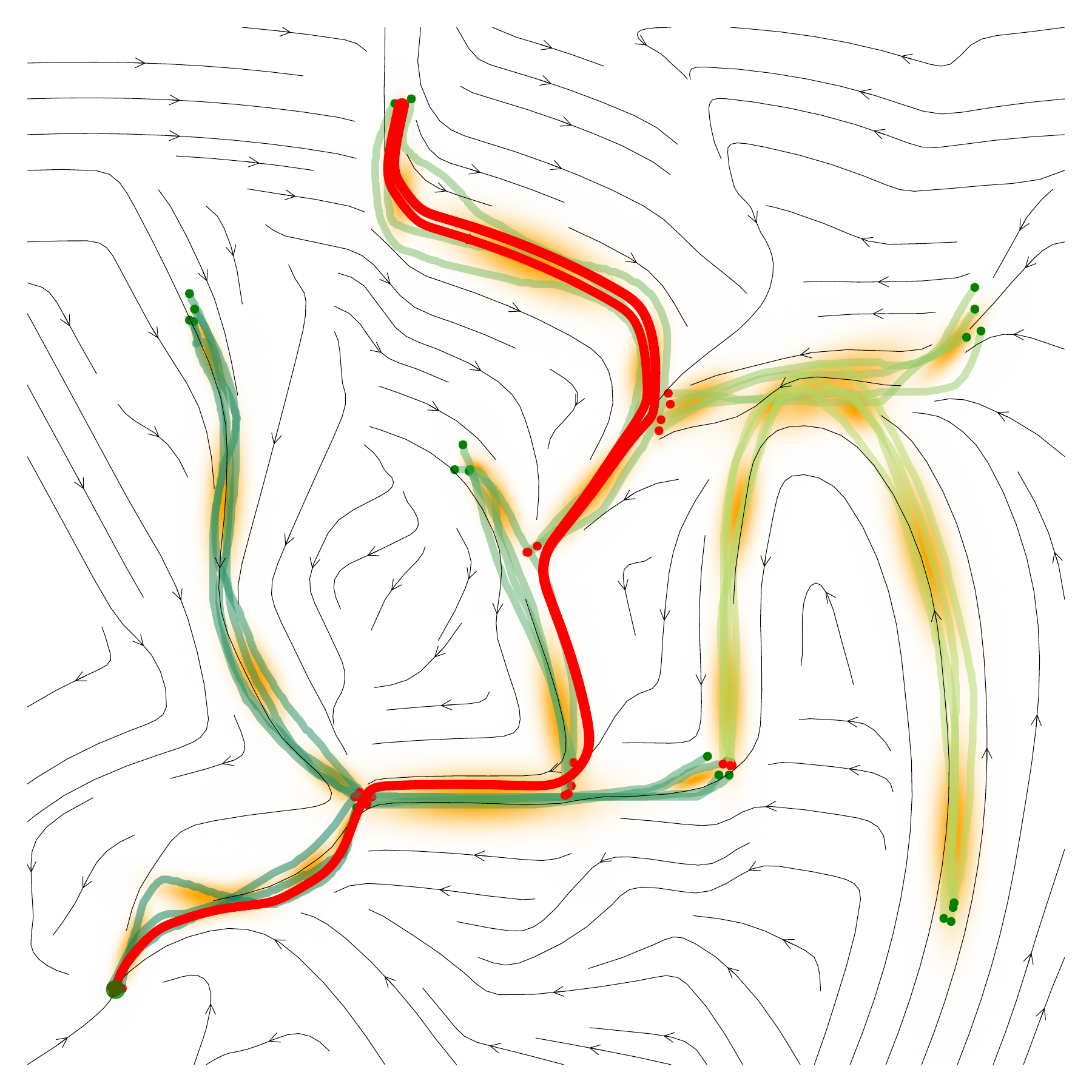}
        \caption{Stitch-SPT (DS)}
        \label{fig:methods_comparison_i}
    \end{subfigure}
    \hfill
    \begin{subfigure}[t]{0.193\textwidth}
        \centering
        \includegraphics[width=\textwidth]{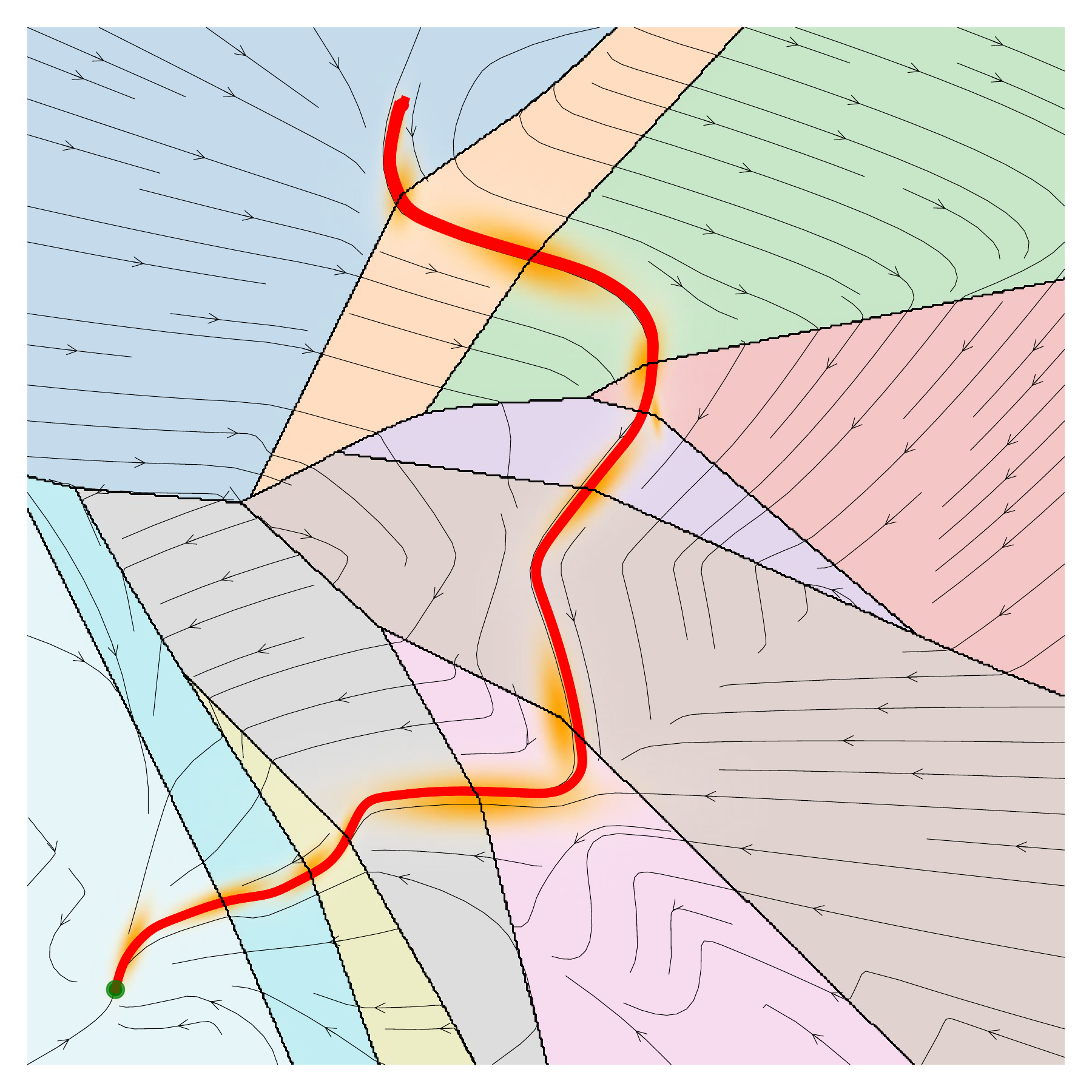}
        \caption{Chaining (All)}
        \label{fig:}
    \end{subfigure}
    \begin{subfigure}[t]{0.193\textwidth}
        \centering
        \includegraphics[width=\textwidth]{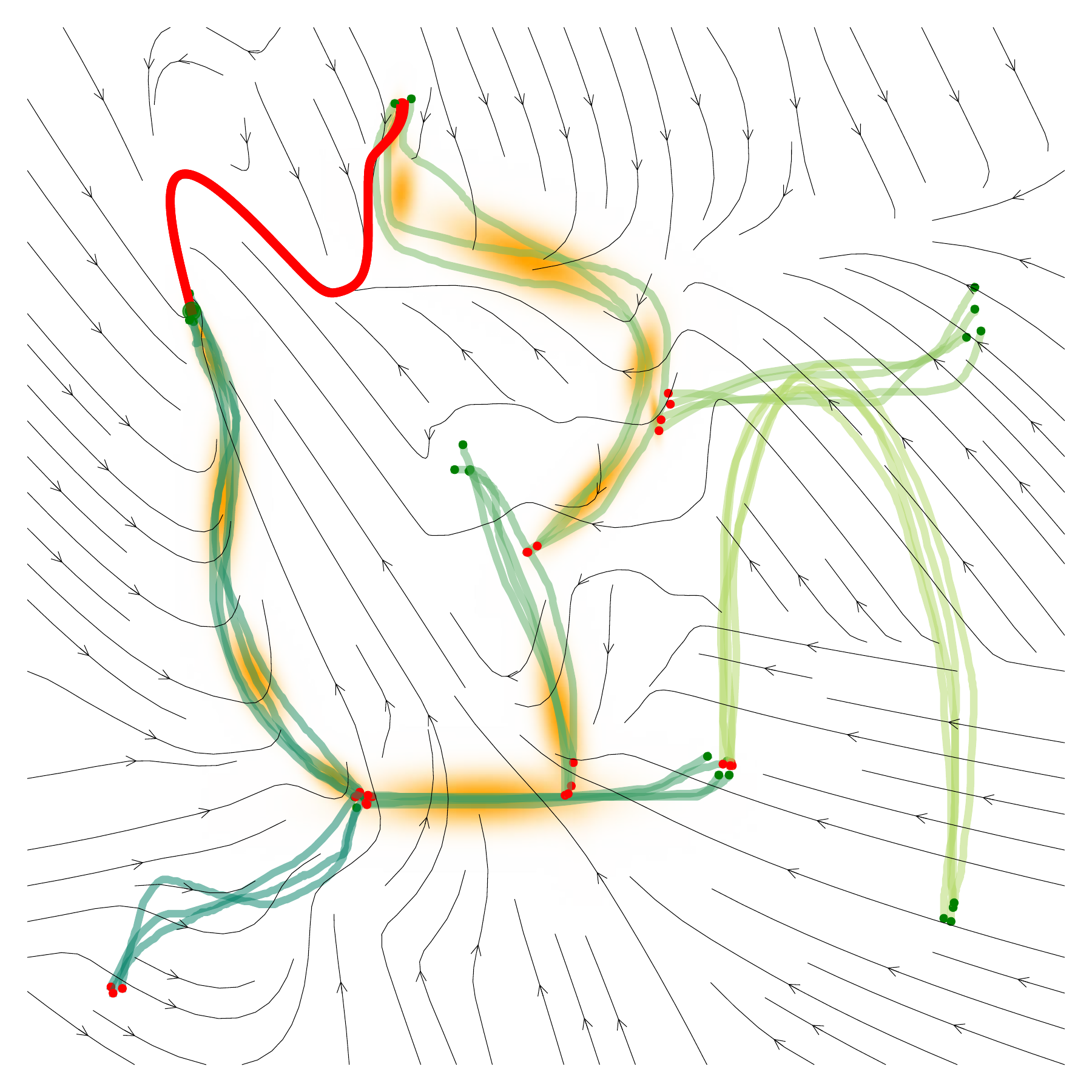}
        \caption{Stitch-SP (All)}
        \label{fig:methods_comparison_k}
    \end{subfigure}
    \hfill
    \begin{subfigure}[t]{0.193\textwidth}
        \centering
        \includegraphics[width=\textwidth]{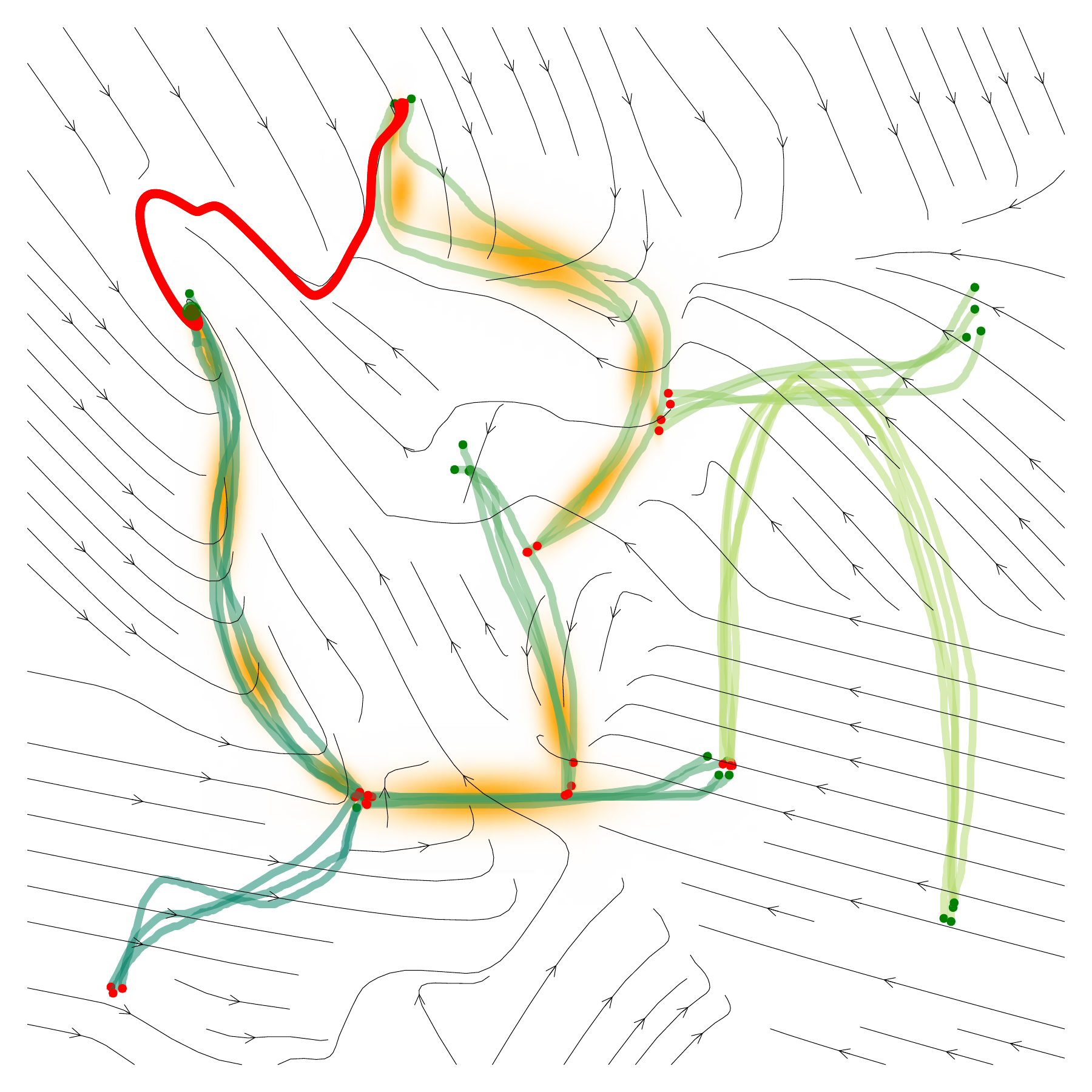}
        \caption{Stitch-SP (DS)}
        \label{fig:}
    \end{subfigure}
    \hfill
    \begin{subfigure}[t]{0.193\textwidth}
        \centering
        \includegraphics[width=\textwidth]{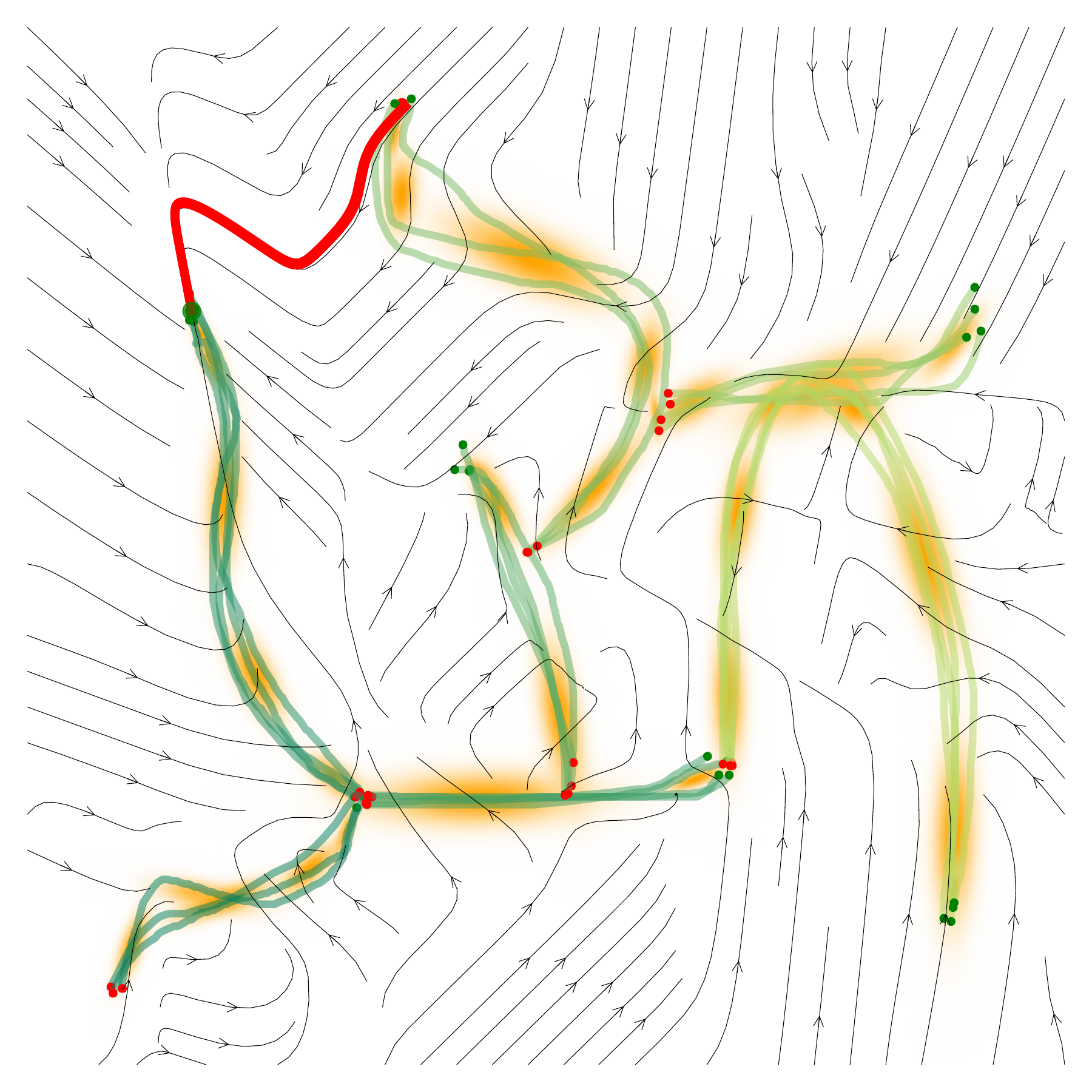}
        \caption{Stitch-SPT (All)}
        \label{fig:}
    \end{subfigure}
    \hfill
    \begin{subfigure}[t]{0.193\textwidth}
        \centering
        \includegraphics[width=\textwidth]{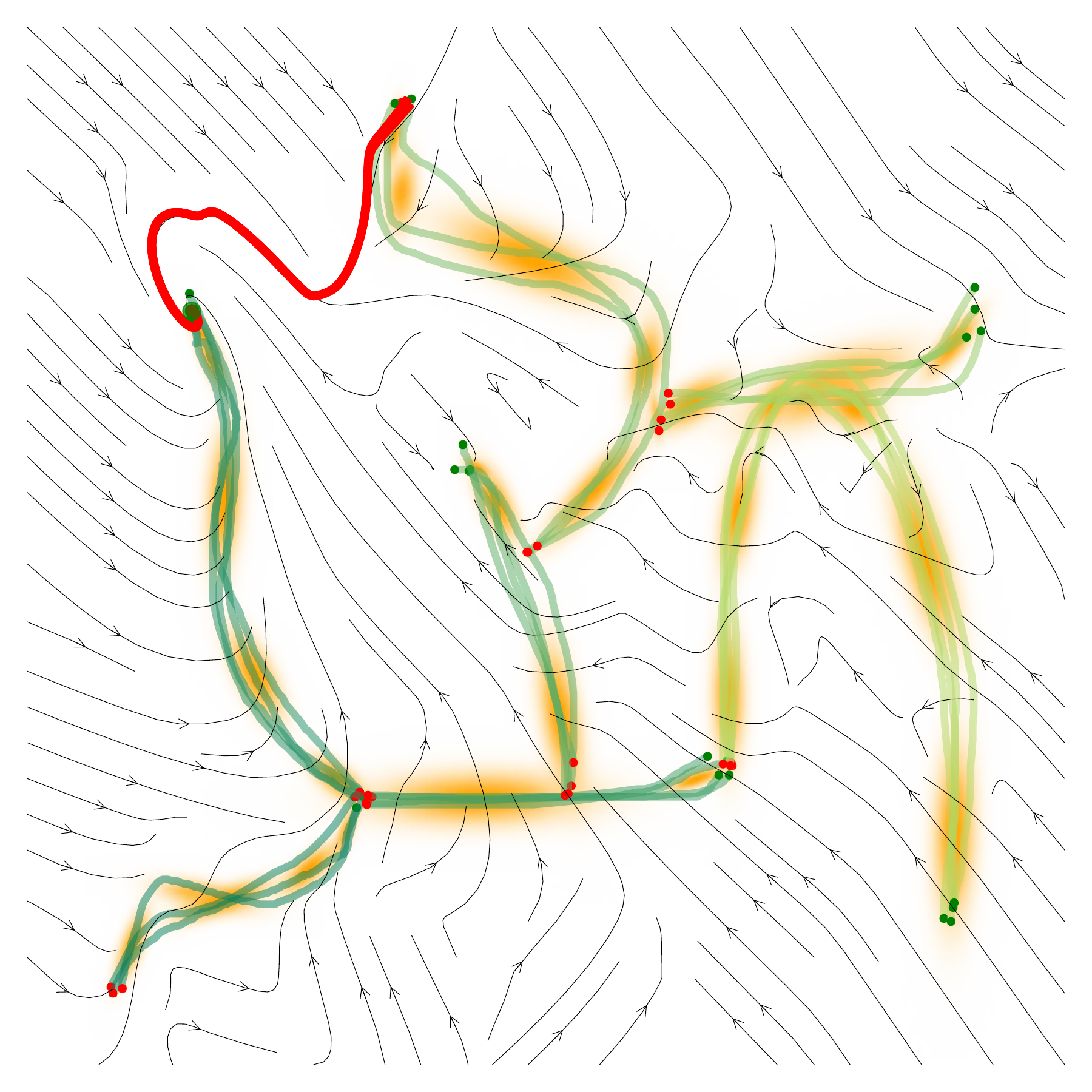}
        \caption{Stitch-SPT (DS)}
        \label{fig:}
    \end{subfigure}
    \hfill
    \begin{subfigure}[t]{0.193\textwidth}
        \centering
        \includegraphics[width=\textwidth]{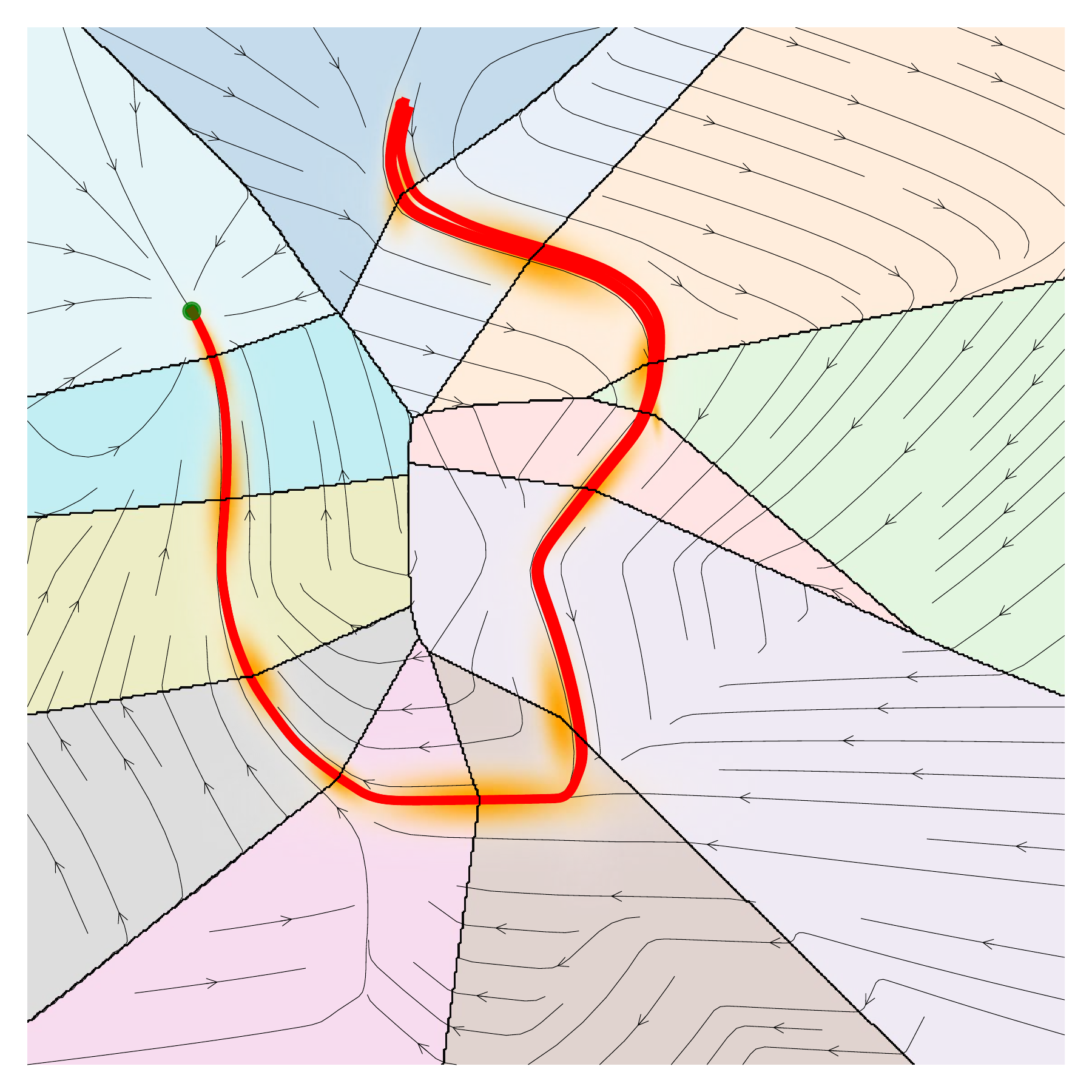}
        \caption{Chaining (All)}
        \label{fig:methods_comparison_o}
    \end{subfigure}
    \caption{Qualitative results for datasets \datasetSmall (first row) and \datasetLarge (second and third row). Simulations are shown in red, demonstrations in blue-green, and Gaussians extracted from the GG in orange. Baseline methods are excluded as they failed in all the presented cases, as well as Chaining (DS) since results are similar to Chaining (All).}
    \label{fig:methods_comparison}
    \vspace{-3mm}
\end{figure*}

\subsection{Constructing a GAS DS-Chain using the Gaussian Graph}

We now define our specific method of creating a DS-Chain using the GG, showing that it satisfies every criterion of Theorem~\ref{theorem:DSC_GAS} and is therefore GAS.

\subsubsection{Dynamical systems in $\DSseq$}

From the Gaussian graph $\Graph$, a sequence of vertices $\ggSolution = \langle v_1, v_2, \dots, v_{M} \rangle$ can be obtained by, for instance, finding the shortest path from an initial position $\pos_0$ to $\pos^*$. Each pair $(v_i,v_{i+1})$ is an edge in $\Edges$.
Just as in Section~\ref{sec:shortest_path}, we do not include the vertices $v_0$ and $v^*$ corresponding to $\pos_0$ to $\pos^*$ in $\ggSolution$.

Our approach is to partition $\ggSolution$ using a sliding window to fit a local DS to each segment. 
While the simplest approach uses vertex pairs (i.e., single edges), we achieve better results using overlapping vertex triplets.
Triples are advantageous because adjacent segments share exactly one edge, providing a continuous transition region between successive local models.
This yields a sequence of $M-2$ overlapping segments: 
$\langle \langle v_1, v_2, v_3 \rangle, \langle v_2, v_3, v_4 \rangle, \dots, \langle v_{M-2}, v_{M-1}, v_M \rangle \rangle$.
We make one assumption on the path $\ggSolution$: no segment derived from $\ggSolution$ contains any vertex \emph{and} its reversal (under the bi-directionality assumption) due to the difficulty of finding a DS from the resulting set of data points.

LPV-DS is applied to each segment $\langle v_i, v_{i+1}, v_{i+2} \rangle$ to obtain an LTI DS $f_i$ that is GAS at the final vertex's position, $\gmean_{i+2}$.
While it seems intuitive to fit $f_i$ to all data points clustered to these three vertices, 
we find that LPV-DS more reliably obtains a DS when omitting
a specific subset of data around the target vertex $v_{i+2}$.
Specifically, let $l = ||\gmean_{i+2} - \gmean_{i+1}||_2$ be the distance between the last two vertices. 
To improve fitting, we discard any reference point $\pos^\mathit{ref}$ belonging to $v_{i+2}$ that is both farther from $v_{i+1}$ than $l$ (i.e., $||\pos^\mathit{ref} - \gmean_{i+1}||_2 > l$)
and closer to $v_{i+2}$ than $0.1l$ (i.e., $||\pos^\mathit{ref} - \gmean_{i+2}||_2 < 0.1l$).
Observe that $f_i$ is invariant to $\pos_0$ and $\pos^*$ and can therefore be pre-computed and stored in a lookup table (bounded by $|\Edges| \max_{v\in\Nodes}|\text{neighbors}(v)|$ entries).

A final DS $f^*$ is fitted to data from the last two vertices $v_{M-1}$ and $v_{M}$, and ensured GAS with respect to the attractor $\pos^*$.
This results in $\DSseq = \langle f_1,\dots,f_{M-2}, f^* \rangle$.
As an optional step, an initial DS $f_0$ --- GAS with respect to $\gmean_1$ and fitted to data from $v_1$ --- can be prepended to $\DSseq$.
The same levels of reuse as described in Section~\ref{sec:stitch} also apply here.


\subsubsection{Triggers in $\TriggerSeq$}
We define a trigger $\trigger_i(\pos)$ (associated with a segment $\langle v_i, v_{i+1}, v_{i+2}\rangle$) as $\top$ when $\pos$ has passed the second vertex $v_{i+1}$ on its way from $v_i$ to $v_{i+2}$. That is, $\trigger_i(\pos)=\top$ when
\begin{equation}
        \frac{||\pos-\gmean_{i}||_2}{||\pos-\gmean_{i+2}||_2} \geq \frac{||\gmean_{i+1} - \gmean_{i}||_2}{||\gmean_{i+1} - \gmean_{i+2}||_2},
        \label{eq:trigger}
\end{equation} 
$\trigger_i(\pos)=\perp$ otherwise.
Specifically, for the optional initial DS $f_0$, in~\eqref{eq:trigger} we let 
$\gmean_i=\pos_0$, 
$\gmean_{i+2} = \gmean_1$, and 
$\gmean_{i+1}=(\pos_0 + \gmean_1)/2$.

\subsubsection{Timers in $\TransSeq$}

For each segment $i=1,\dots,N-1$, 
the timer $\trans_i(t)$ switches to $\top$ after prescribed transition time $T_i$.
Specifically, $\trans_i(t)=\top$ when $t-t_i \geq T_i$, $\trans_i(t)=\perp$ otherwise,
where $t$ is the current time and $t_i$ is the time at which the DS-Chain enters state $s_i$.
We choose $T_i\in\Real^+$ to approximate the time required for the state $\pos$ to travel a fraction $\alpha\in[0,1]$ from $v_{i+1}$ to $v_{i+2}$:
\begin{equation}
    T_i = \alpha \frac{|| \gmean_{i+2} - \gmean_{i+1} ||_2}{||f_i(\gmean_{i+1})+f_{i+1}(\gmean_{i+1})||_2 / 2}
\end{equation}
This can be interpreted as follows:
state $s_i$ is entered when trigger $\trigger_i=\top$,
which occurs near the middle Gaussian $i+1$ of the segment $\langle v_i, v_{i+1}, v_{i+2} \rangle$.
Hence, $T_i$ estimates the time needed to travel a fraction $\alpha$ of the distance from $v_{i+1}$ to $v_{i+2}$ under the average velocity of the current and next DS.
In particular, $\alpha=0$ yields an immediate transition, whereas $\alpha=1$ delays the transition until approximately the full segment has been traversed ($\pos\approx \gmean_{i+2}$).


\begin{table*}[]
    \centering 
    \resizebox{\textwidth}{!}{%
        \begin{tabular}{l | llll | llll | llll}
        \toprule
          & \multicolumn{4}{c|}{\textbf{2D Small}} & \multicolumn{4}{c|}{\textbf{2D Large}} & \multicolumn{4}{c}{\textbf{3D PC-GMM}} \\
         Method & Success \% & RMSE & Data Support & Comp. Time & Success \% & RMSE & Data Support & Comp. Time & Success \% & RMSE & Data Support & Comp. Time \\
        \midrule
        Baseline (All) & 20.8\% & 1.77 ± 0.72 & 0.73 ± 0.17 & 8.29 ± 0.63 & 19.1\% & 2.15 ± 0.74 & 0.72 ± 0.23 & 17.51 ± 1.86 & 28.3\% & 0.28 ± 0.33 & 0.88 ± 0.17 & 17.78 ± 16.59 \\
        Baseline (DS) & 20.0\% & 2.97 ± 1.57 & 0.68 ± 0.25 & 10.16 ± 0.42 & 13.7\% & 3.30 ± 0.76 & 0.58 ± 0.24 & 33.65 ± 2.67 & 35.8\% & 2.53 ± 1.19 & 0.93 ± 0.10 & 251 ± 23 \\
        Stitch-SP (All) & 98.3\% & \textbf{0.08 ± 0.02} & 0.86 ± 0.24 & 0.76 ± 0.42 & 90.2\% & 0.14 ± 0.30 & 0.90 ± 0.17 & 1.09 ± 1.28 & 90.0\% & \textbf{0.04 ± 0.03} & 0.87 ± 0.19 & 1.33 ± 1.27 \\
        Stitch-SP (DS) & 94.2\% & \textbf{0.08 ± 0.02} & 0.86 ± 0.24 & 0.14 ± 0.71 & 91.3\% & 0.19 ± 0.41 & 0.87 ± 0.21 & 0.59 ± 1.90 & 85.8\% & \textbf{0.04 ± 0.02} & 0.85 ± 0.20 & 0.06 ± 0.04 \\
        Stitch-SPT (All) & 97.5\% & 0.09 ± 0.03 & 0.86 ± 0.19 & 1.44 ± 0.72 & 89.7\% & 0.27 ± 0.44 & 0.77 ± 0.22 & 5.06 ± 5.95 & 81.7\% & 0.08 ± 0.11 & 0.92 ± 0.13 & 7.10 ± 12.37 \\
        Stitch-SPT (DS) & 95.8\% & \textbf{0.08 ± 0.02} & 0.85 ± 0.20 & 0.08 ± 0.02 & 91.6\% & 0.67 ± 0.82 & 0.80 ± 0.22 & 15.28 ± 16.46 & \textbf{100.0\%} & 0.45 ± 0.57 & 0.95 ± 0.06 & 196 ± 100 \\
        Chaining (All) & \textbf{99.2\%} & 0.11 ± 0.14* & \textbf{0.91 ± 0.14} & 0.26 ± 0.05 & \textbf{97.4\%} & \textbf{0.12 ± 0.18}* & \textbf{0.96 ± 0.09} & 0.29 ± 0.09 & 97.5\% & 0.05 ± 0.03* & 0.97 ± 0.06 & 0.46 ± 0.32 \\
        Chaining (DS) & 97.5\% & 0.14 ± 0.19* & \textbf{0.91 ± 0.17} & \textbf{0.01 ± 0.01} & 94.1\% & 0.15 ± 0.39* & 0.95 ± 0.09 & \textbf{0.01 ± 0.00} & \textbf{100.0\%} & 0.06 ± 0.02* & \textbf{0.98 ± 0.03} & \textbf{0.02 ± 0.00} \\
        \bottomrule
        \end{tabular}
    }
    \caption{The results on each dataset for the methods LPV-DS (Baseline), Stitching using the Shortest Path (Stitch SP) and Shortest Path Tree (Stitch SPT), and Chaining, showing the success rate in percent, RMSE, Data Support, and computation time in seconds from four random seeds. 
    The two levels of reuse from Section~\ref{sec:stitch} --- recomputing both the GMM and DS parameters (all) or only recomputing the DS parameters (DS) --- are tested. $^*$Calculating the RMSE of \emph{Chaining} methods is approximated by mixing DSs equally during transitions. 
    }
    \label{tab:results}
    \vspace{-3mm}
\end{table*}

\subsubsection{Intermediate dynamics in $\InterDSseq$}
The system remains in intermediate state $s_i'$ for $T_i$ time while transitioning from state $s_i$ with dynamics $f_i$ to state $s_{i+1}$ with dynamics $f_{i+1}$.
Therefore, we let the dynamics $\inter_i$ during this transition be a linear interpolation between the two DSs:
\begin{equation}
    \inter_i(\pos, t) = \frac{t-t_i}{T_i}f_i(\pos) + \left(1 - \frac{t-t_i}{T_i}\right)f_{i+1}(\pos).
\end{equation}
Since $f_1,\dots,f_N$ are all bounded, so too are the linear interpolations $u_1, \dots, \inter_{N-1}$.

By construction, $\DSC = \langle \DSseq, \TriggerSeq, \TransSeq, \InterDSseq \rangle$ satisfies all criteria of Theorem~\ref{theorem:DSC_GAS} and is therefore GAS with respect to $\pos^*$.
Criteria~\ref{theorem:DSC_GAS:criteria_1} and \ref{theorem:DSC_GAS:criteria_4} hold because all primary and intermediate dynamics in $\DSseq$ and $\InterDSseq$ are bounded, and the final system $f_N = f^*$ is GAS.
Furthermore, Criteria~\ref{theorem:DSC_GAS:criteria_2} and~\ref{theorem:DSC_GAS:criteria_3} are satisfied because every trigger $\trigger_i(\pos)$ evaluates to $\top$ in finite time (since each $f_i$ is GAS towards $\gmean_{i+2}$) and the transition timers $\trans_i(t)$ rely on finite $T_i$.

%% file: sections/Experiments.tex
\subsection{Simulation Experiments}
We evaluate the proposed methods alongside LPV-DS as a baseline on three datasets:
\datasetSmall(Fig.~\ref{fig:pipeline} and Figs.~\ref{fig:methods_comparison_a}--\ref{fig:methods_comparison_e});
a larger
\datasetLarge of 6 demonstrations (Figs.~\ref{fig:methods_comparison_f}--\ref{fig:methods_comparison_o}); and
\datasetPcgmm based on the dataset from~\cite{figueroa2018physically}, comprising ``3D C-shape top'', ``3D viapoint 1'', and ``3D viapoint 2''.
For each dataset, all pairwise combinations of pooled initial-attractor points (assuming bi-directionality) yield $30$ instances for \datasetSmall, $182$ for \datasetLarge, and $30$ for \datasetPcgmm. 
All experiments use $\parambhat=0.05$, $\paramdist=2$ and $\paramdir=1$, thus slightly prioritizing closeness over alignment, and are rerun using four random seeds.

Results are presented in Table~\ref{tab:results}.
The \emph{success-rate} is shown, with a success meaning that a valid solution was found and the goal state was reached within $1000$ simulation seconds.
\emph{Root Mean-Squared Error} (RMSE) is taken between $\vel=f(\pos^\mathit{ref})$ and $\vel^\mathit{ref}$.
Due to these methods aggregating reference points from different parts of various demonstrations, we found the common \emph{Dynamic Time Warping Distance} to be an unsuitable measurement of how data-supported a trajectory is.
Instead, a trajectory's \emph{Data Support} is the average evaluation of its points. 
A point with distance $d$ to the nearest reference point evaluates to $1$ if $d<\mu_\DemoSet$, else $\exp\left(-\frac{1}{2}\left(\frac{d-\mu_\DemoSet}{\sigma_\DemoSet}\right)^2\right)$, where $\mu_\DemoSet$ and $\sigma_\DemoSet$ are the mean and standard deviation of inter-trajectory nearest-neighbor distances over reference points in demonstrations in $\DemoSet$. 
Thus, a Data Support of $1.0$ indicates the trajectory to be as close to the data as the reference trajectories, while $0.0$ indicates no support.
Metrics are averaged over successful solutions. 
More details are available on our repository\footref{footnote:repo}.

As shown in Table~\ref{tab:results}, our proposed GG methods successfully synthesize new trajectories, achieving substantially higher success rates across all datasets compared to the LPV-DS baselines, which fail in the majority of cases.
Furthermore, in the instances where valid solutions are found, our approaches generally yield higher-quality trajectories. 
This is evidenced by improved alignment with reference velocities (lower RMSE) and stronger adherence to demonstrated regions (higher Data Support) compared to baseline methods.
Our proposed methods additionally yield lower inference times, with \emph{Chaining (DS)} outperforming the others by a significant margin. 
This rapid online performance is achieved by shifting much of the computational load to offline precomputations, which averages $4.2$ seconds for \emph{Chaining (DS)} and $44.3$ seconds for \emph{Chaining (All)}.

Fig.~\ref{fig:methods_comparison} illustrates qualitative examples of our proposed methods. 
Notably, the structural difference between the \emph{Stitch-SP} and \emph{Stitch-SPT} methods is visible: 
\emph{Stitch-SP} yields a DS tailored to a specific initial position, 
whereas \emph{Stitch-SPT} remains agnostic to the initial position and thereby generates a global DS toward the target. 
The bottom row (Figs.~\ref{fig:methods_comparison_k}--\ref{fig:methods_comparison_o}) highlights the particular strengths of \emph{Chaining} in complex scenarios where \emph{Stitching} fails.
Specifically, \emph{Chaining} is able to formulate and execute sophisticated trajectories that are otherwise challenging for a single time-invariant DS to represent while satisfying the strict decay constraints of a Lyapunov function. 

Overall, the simulation results confirm that both \emph{Stitching} and \emph{Chaining} effectively enable zero-shot generalization to unseen tasks with a relatively high rate of success. 

\subsection{Real-Robot Experiments}

\begin{figure*}[t]
    \centering
    \includegraphics[width=\textwidth]{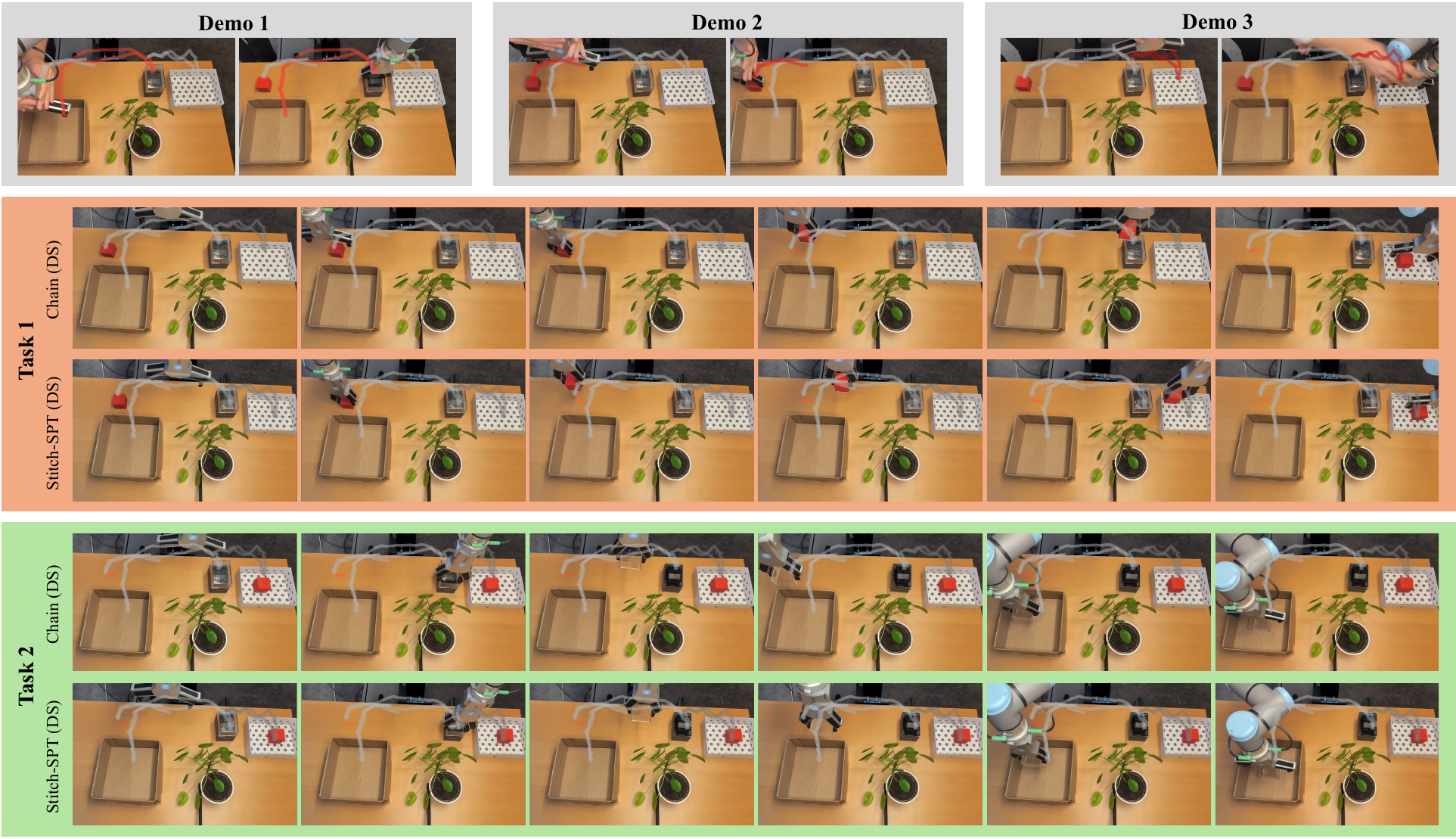}
    \caption{Real-robot experiments where we collect 3 demonstrations with 2 trajectories each and test \emph{Stitch-SPT (DS)} and \emph{Chaining (DS)} on two different tasks that require reversing and stitching. The environment consists of a transparent cube (object 1), a red cube (object 2), a brown box (box A), a white box (box B), and a plant acting as an obstacle.}
    \label{fig:real_robot}
\end{figure*}

Furthermore, we validate our simulated findings through real-robot experiments using the two methods with the highest success rate on \datasetPcgmm: \emph{Stitch-SPT (DS)} and \emph{Chaining (DS)}. The experiments are conducted on a UR3e manipulator equipped with an OnRobot two-finger gripper, using the same configurations as in the previous section.
We collect training data for three distinct demonstrations (two trajectories each): (i) grasping object 1 starting from box A, (ii) grasping object 2, and (iii) transitioning from above object 1 to box B (see first row of Fig. \ref{fig:real_robot}). Importantly, grasping solely refers to moving towards the object and closing the gripper without lifting it.

Using this dataset, we evaluate zero-shot generalization on two unseen composite tasks, both starting from the workspace center: (1) picking up object 2 and moving it to box B, and (2) picking object 1 and moving it to box A. 
In order to solve task (1), the underlying policy needs to combine all the existing demonstrations. Conversely, task (2) requires the system to reverse demonstration (i) without experiencing interference from other data. As the DS outputs end-effector velocities, we map them to joint velocities via an inverse-kinematics solver and execute them with a velocity controller. 
Gripper actions are triggered when reaching the desired pickup location.
Additionally, a complete task is defined by consecutive attractor positions.
For instance, in task (1), the attractor is initially object 2, but then changed to box B once object 2 is grasped.

Fig. \ref{fig:real_robot} illustrates the resulting behaviors for both \emph{Stitch-SPT (DS)} and \emph{Chaining (DS)}. As can be seen, both methods successfully complete the tasks, confirming the efficacy of our methods as observed in the simulation datasets. Crucially, integrating trajectory stitching (or chaining) combined with the bi-directionality assumption enables the end-effector to freely move around the whole workspace while remaining within the demonstrated data manifold, allowing it to closely follow the most fitting demonstration at any given time.

%% file: sections/Discussions.tex
\subsection{Methodological Trade-Offs}

Each of the proposed methods offers distinct advantages while presenting specific trade-offs.
\emph{Stitch-SP} constructs a DS tailored to a specific initial-target pair; 
it generally achieves trajectory qualities comparable to \emph{Stitch-SPT}, but requires less computation time.
On the other hand, \emph{Stitch-SPT} generates a global DS dependent exclusively on the target position, resulting in a policy that generally accommodates any initial position within data support and thereby making it inherently more robust to large perturbations. 
Furthermore, it can be entirely precomputed if the set of desired target positions is known in advance.

\emph{Chaining} yields the highest-quality trajectories and extremely fast computation times, provided the offline precomputation is completed.
Crucially, \emph{Chaining} is the only method among those evaluated capable of producing highly sophisticated, intersecting trajectories (such as those in Fig.~\ref{fig:chaining_example} and Figs.~\ref{fig:methods_comparison_k}--\ref{fig:methods_comparison_o}) that exceed the representational capacity of a single time-invariant DS. 
However, because \emph{Chaining} formulates an inherently time-varying DS, it relies on the implicit assumption that the system state remains relatively close to the active graph segment.
This results in a higher sensitivity to large perturbations compared to the \emph{Stitching} approaches.
Ultimately, the optimal choice of method is application-dependent, balancing the need for robustness, computational performance, and trajectory complexity.

\subsection{Incremental Learning and Scalability}

One of the most significant advantages of the proposed methods is their scalability to unseen tasks within a shared workspace.
Rather than collecting a new demonstration for every possible task, existing demonstrations can be reused to synthesize novel trajectories.
If an entirely new region of the workspace must be reached, only a brief, localized demonstration connecting the new area to the existing spatial network is required.
Thus, our methods integrate naturally with incremental learning. 
Incorporating a new demonstration requires fitting an LPV-DS model to the new demonstration, translating it into new vertices and edges, and appending them to the existing GG.
Once updated, \emph{Stitching} or \emph{Chaining} can be applied. 
Notably, prior \emph{Chaining} precomputations can be retained;
the system only needs to compute and store the localized dynamics of newly added segments.
This flexibility is key for deployment in real-life settings where adjustability and extendibility are critical.

\subsection{Limitations and Future Work}

While the proposed framework successfully enables zero-shot generalization, it presents a few limitations --- beyond those already discussed --- that motivate future work.
First, the construction of the GG currently relies on user-defined hyperparameters ($\parambhat, \paramdist, \paramdir$) to balance data-overlap, spatial distance, and dynamic consistency. 
Besides exploring the wide range of alternative edge-weight formulations, future work could aim to automate the parameter selection process.
This could be achieved by framing the graph connectivity as an optimization problem, or by learning acceptable parameter bounds directly from the demonstration data distribution.

Second, our approach aims to synthesize novel trajectories within data support. 
However, it does not explicitly account for changing object positions or obstacles introduced into the workspace after the demonstrations were collected. 
Integrating real-time obstacle avoidance --- for instance, using DS \emph{modulation}~\cite{DSmodulation} or \emph{Control Barrier Functions}~\cite{CBF} --- could be implemented both during the GG path-finding phase and within the resulting control policy. Similarly, combining our method with, for instance, \emph{Elastic Motion Policies}~\cite{li2025elastic} would allow us to adapt to different object positions.

Finally, the current implementation operates primarily on spatial components (Euclidean space). 
A highly promising avenue for future work is extending the approach to include end-effector orientations by integrating SE(3) LPV-DS~\cite{sun2024se}. 
The primary modification required would be adapting the edge-weight metrics used during the graph search to account for rotational manifolds.
This extension highlights the underlying generality of the Gaussian Graph framework and provides a strong theoretical foundation for scaling to more complex robotics tasks.

%% file: sections/Conclusion.tex
Although there are multiple ways to learn motion policies from expert demonstrations, dynamical system (DS) based methods have proven to be data-efficient and provide fast and reactive control with stability guarantees.
However, these methods often lack the ability to combine learned skills to generalize to new scenarios. 
In this work, we presented a framework to effectively 
perform zero-shot generalization from motion demonstrations of separate tasks in a shared workspace to new, undemonstrated tasks.
To this end, we introduced the \emph{Gaussian Graph} (GG), constructed from the components obtained by applying Linear Parameter-Varying DS (LPV-DS) to each demonstration separately. 
The GG thereby allows for the unification of continuous control with well-known graph algorithms.
We then introduced two classes of methods:
\emph{Demonstration Stitching}, which combines GG components along a shortest path or in a shortest path tree to produce a time-invariant DS; and
\emph{Demonstration Chaining}, which divides a path in the GG into ``stepping stones'' that are sequentially traversed.
The Stitching methods naturally inherit LPV-DS's stability guarantees while reliably generalizing to new tasks, 
while for Chaining we provide formal proofs of stability and showcase the flexibility it provides in the resulting types of motions. 
Simulations and real-robot evaluations show their efficacy for zero-shot generalization where baselines fail.
Overall, this work contributes towards developing methods that are sample-efficient, reactive, and provably convergent while effectively reusing data for unseen tasks.

%% file: Bib.bib
@inproceedings{figueroa2018physically,
  title={A Physically-Consistent Bayesian Non-Parametric Mixture Model for Dynamical System Learning},
  author={Figueroa, Nadia and Billard, Aude},
  booktitle={Conference on Robot Learning},
  pages={927--946},
  year={2018},
  organization={PMLR}
}

@article{kawaharazuka2024real,
  title={Real-world robot applications of foundation models: A review},
  author={Kawaharazuka, Kento and Matsushima, Tatsuya and Gambardella, Andrew and Guo, Jiaxian and Paxton, Chris and Zeng, Andy},
  journal={Advanced Robotics},
  volume={38},
  number={18},
  pages={1232--1254},
  year={2024},
  publisher={Taylor \& Francis}
}

@book{billard2022learning,
  title={Learning for Adaptive and Reactive Robot Control: A Dynamical Systems Approach},
  author={Billard, Aude and Mirrazavi, Sina and Figueroa, Nadia},
  year={2022},
  publisher={Mit Press}
}

@article{kim2024openvla,
  title={{OpenVLA}: An Open-Source Vision-Language-Action Model},
  author={Kim, Moo Jin and Pertsch, Karl and Karamcheti, Siddharth and Xiao, Ted and Balakrishna, Ashwin and Nair, Suraj and Rafailov, Rafael and Foster, Ethan and Lam, Grace and Sanketi, Pannag and others},
  journal={arXiv preprint arXiv:2406.09246},
  year={2024}
}

@inproceedings{zitkovich2023rt,
  title={{RT-2}: Vision-Language-Action Models Transfer Web Knowledge to Robotic Control},
  author={Zitkovich, Brianna and Yu, Tianhe and Xu, Sichun and Xu, Peng and Xiao, Ted and Xia, Fei and Wu, Jialin and Wohlhart, Paul and Welker, Stefan and Wahid, Ayzaan and others},
  booktitle={Conference on Robot Learning},
  pages={2165--2183},
  year={2023},
  organization={PMLR}
}

@inproceedings{sun2024se,
  title={{SE(3)} Linear Parameter Varying Dynamical Systems for Globally Asymptotically Stable End-Effector Control},
  author={Sun, Sunan and Figueroa, Nadia},
  booktitle={2024 IEEE/RSJ International Conference on Intelligent Robots and Systems (IROS)},
  pages={5152--5159},
  year={2024},
  organization={IEEE}
}

@incollection{dijkstra,
  title={A Note on Two Problems in Connexion with Graphs},
  author={Dijkstra, Edsger W},
  booktitle={Edsger Wybe Dijkstra: his life, work, and legacy},
  pages={287--290},
  year={2022}
}

@inproceedings{medina2017learning,
  title={Learning Stable Task Sequences from Demonstration with Linear Parameter Varying Systems and Hidden Markov Models},
  author={Medina, Jose R and Billard, Aude},
  booktitle={Conference on Robot Learning},
  pages={175--184},
  year={2017},
  organization={PMLR}
}

@INPROCEEDINGS{li2025elastic,
  author={Li, Tianyu and Sun, Sunan and Aditya, Shubhodeep Shiv and Figueroa, Nadia},
  booktitle={2025 IEEE/RSJ International Conference on Intelligent Robots and Systems (IROS)}, 
  title={Elastic Motion Policy: An Adaptive Dynamical System for Robust and Efficient One-Shot Imitation Learning}, 
  year={2025},
  volume={},
  number={},
  pages={9846-9853},
  doi={10.1109/IROS60139.2025.11246142}
  }

@article{bhattacharyya1943measure,
  title={On a measure of divergence between two statistical populations defined by their probability distribution},
  author={Bhattacharyya, Anil},
  journal={Bulletin of the Calcutta Mathematical Society},
  volume={35},
  pages={99--110},
  year={1943}
}

@Article{Calinon2016,
author={Calinon, Sylvain},
title={A tutorial on task-parameterized movement learning and retrieval},
journal={Intelligent Service Robotics},
year={2016},
month={Jan},
day={01},
volume={9},
number={1},
pages={1-29},
issn={1861-2784},
doi={10.1007/s11370-015-0187-9},
url={https://doi.org/10.1007/s11370-015-0187-9}
}

@INPROCEEDINGS{CBF,
  author={Ames, Aaron D. and Coogan, Samuel and Egerstedt, Magnus and Notomista, Gennaro and Sreenath, Koushil and Tabuada, Paulo},
  booktitle={2019 18th European Control Conference (ECC)}, 
  title={Control Barrier Functions: Theory and Applications}, 
  year={2019},
  volume={},
  number={},
  pages={3420-3431},
  doi={10.23919/ECC.2019.8796030}}

@Article{DSmodulation,
    author={Khansari-Zadeh, Seyed Mohammad
    and Billard, Aude},
    title={A dynamical system approach to realtime obstacle avoidance},
    journal={Autonomous Robots},
    year={2012},
    month={May},
    day={01},
    volume={32},
    number={4},
    pages={433-454},
    issn={1573-7527},
    doi={10.1007/s10514-012-9287-y},
    url={https://doi.org/10.1007/s10514-012-9287-y}
}
